\documentclass[twoside,11pt]{article}

\usepackage{multirow}
\usepackage{array}
\usepackage{tabularx}
\newcolumntype{H}{>{\setbox0=\hbox\bgroup}c<{\egroup}@{}} 
\newcolumntype{L}[1]{>{\raggedright\let\newline\\\arraybackslash\hspace{0pt}}m{#1}}
\newcolumntype{C}[1]{>{\centering\let\newline\\\arraybackslash\hspace{0pt}}m{#1}}
\newcolumntype{R}[1]{>{\raggedleft\let\newline\\\arraybackslash\hspace{0pt}}m{#1}}

\usepackage{pbox}

\usepackage{amsmath}
\usepackage{amssymb}

\newtheorem{defi}{Definition}

\usepackage{thmtools}
\usepackage{thm-restate}

\usepackage{tikz}
\usetikzlibrary{shapes,backgrounds}
\usepackage{verbatim}
\usepackage{pgfplots}

\usepgflibrary{shapes.geometric}
\usetikzlibrary{calc}

\usepackage{sidecap}
\sidecaptionvpos{figure}{t}
\usepackage{tkz-kiviat}
\usetikzlibrary{arrows}

\usepackage{booktabs}

\usepackage{enumitem}

\usepackage{changepage}

\usepackage{graphicx} 
\usepackage[bf,hang,FIGTOPCAP,scriptsize]{subfigure}

\usepackage{algorithm}
\usepackage{algorithmic}

\usepackage{lscape}

\newcommand{\specialcell}[2][l]{%
  \begin{tabular}[#1]{@{}l@{}}#2\end{tabular}}

\usepackage{jmlr2e}

\ShortHeadings{Adjusting for Chance Clustering Comparison Measures}{Romano, Vinh, Bailey, and Verspoor}
\firstpageno{1}

\begin{document}

\title{Adjusting for Chance Clustering Comparison Measures}

\author{\name Simone Romano \email simone.romano@unimelb.edu.au
       \AND 
       \name Nguyen Xuan Vinh \email vinh.nguyen@unimelb.edu.au
       \AND 
       \name James Bailey \email baileyj@unimelb.edu.au 
       \AND
       \name Karin Verspoor \email karin.verspoor@unimelb.edu.au              
\AND       
       \addr Dept.\ of Computing and Information Systems, The University of Melbourne, VIC, Australia.
}

\editor{} 

\maketitle

\thispagestyle{empty} 

\begin{abstract}
Adjusted for chance measures are widely used to compare partitions/clusterings of the same data set. In particular, the Adjusted Rand Index (ARI) based on pair-counting, and the Adjusted Mutual Information (AMI) based on Shannon information theory are very popular in the clustering community. 
Nonetheless it is an open problem as to what are the best application scenarios for each measure and guidelines in the literature for their usage are sparse, with the result that users often resort to using both. 
Generalized Information Theoretic (IT) measures based on the Tsallis entropy have been shown to link pair-counting and Shannon IT measures. In this paper, we aim to bridge the gap between adjustment of measures based on pair-counting and measures based on information theory. We solve the key technical challenge of analytically computing the expected value and variance of generalized IT measures. This allows us to propose adjustments of generalized IT measures, which reduce to well known adjusted clustering comparison measures as special cases. Using the theory of generalized IT measures, we are able to propose the following guidelines for using ARI and AMI as external validation indices: ARI should be used when the reference clustering has large equal sized clusters; AMI should be used when the reference clustering is unbalanced and there exist small clusters.

\end{abstract}


\section{Introduction}

Clustering comparison measures are used to compare partitions/clusterings of the same data set. In the clustering community~\citep{Aggarwal2013}, they are extensively used for external validation when the ground truth clustering is available. A family of popular clustering comparison measures are measures based on pair-counting~\citep{Albatineh2006}.  This category comprises the well known similarity measures Rand Index (RI)~\citep{Rand1971} and the Jaccard coefficient (J)~\citep{Ben2001}.
Recently, information theoretic (IT) measures have been also extensively used to compare partitions~\citep{Strehl2003,Nguyen2010}. Given the variety of different possible measures, it is very challenging to identify the best choice for a particular application scenario~\citep{Wu2009}.

The picture becomes even more complex if adjusted for chance measures are also considered. Adjusted for chance measures are widely used external clustering validation techniques because they improve the interpretability of the results. Indeed, two important properties hold true for adjusted measures: they have constant baseline equal to 0 value when the partitions are random and independent, and they are equal to 1 when the compared partitions are identical. Notable examples are the Adjusted Rand Index (ARI)~\citep{Hubert85} and the Adjusted Mutual Information (AMI)~\citep{Nguyen2009}. It is common to see published research that validates clustering solutions against a reference ground truth clustering with the ARI or the AMI. Nonetheless there are still open problems: \emph{there are no guidelines for their best application scenarios shown in the literature to date and authors often resort to employing them both and leaving the reader to interpret.} 

Moreover, some clustering comparisons measures are susceptible to selection bias: when selecting 
the most similar partition to a given ground truth partition, clustering comparison measures are more likely to select partitions with many clusters~\citep{Romano2014}. In~\cite{Romano2014} it was shown that it is beneficial to perform statistical standardization to IT measures to correct for this bias. In particular, standardized IT measures help in decreasing this bias when the number of objects in the data set is small. Statistical standardization has not been applied to pair-counting measures yet in the literature. We solve this challenge in the current paper, and provide further results about the utility of measure adjustment by standardization.

In this work, we aim to \emph{bridge the gap between the adjustment of pair-counting measures and the adjustment of IT measures}. In~\cite{Furiuchi2006,Simovici2007} it has been shown that generalized IT measures based on the Tsallis $q$-entropy~\citep{Tsallis2009} are a further generalization of IT measures and some pair-counting measures such as RI. In this paper, we will exploit this useful idea to connect ARI and AMI.
Furthermore using the same idea, we can perform statistical adjustment by standardization to a broader class of measures, including pair-counting measures.

A key technical challenge is to analytically compute the expected value and variance for generalized IT measures when the clusterings are random. 
To solve this problem, we propose a technique applicable to a broader class of measures we name $\mathcal{L}_\phi$, which includes generalized IT measures as a special case. This generalizes previous work which provided analytical adjustments for narrower
classes: measures based on pair-counting from the family $\mathcal{L}$~\citep{Albatineh2006}, and measures based on the Shannon's mutual information~\citep{Nguyen2009,Nguyen2010}. Moreover, we define a family of measures $\mathcal{N}_\phi$ which generalizes many clustering comparison measures. For measures belonging to this family, the expected value can be analytically approximated when the number of objects is large. Table \ref{tbl:contribution} summarizes the development of this line of work over the past 30 years and positions our contribution.

In summary, we make the following contributions:
\begin{itemize}
\item We define families of measures for which the expected value and variance can be computed analytically when the clusterings are random;
\item We propose generalized adjusted measures to correct for the baseline property and for selection bias.   This captures existing
well known measures as special cases;
\item We provide insights into the open problem of identifying the best application scenarios for clustering comparison measures, in particular the
application scenarios for ARI and AMI.
\end{itemize}
\begin{figure}[t]
\centering
\includegraphics[scale=1.3]{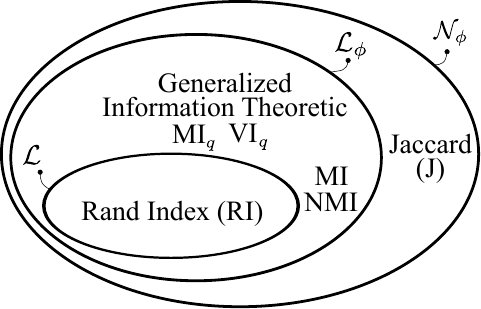}
\caption{Families of clustering comparison measures.} \label{fig:families}
\end{figure}
\begin{table}[t]
\centering
\begin{tabular}{cll}
\toprule
Year & Contribution & Reference \\
\toprule
1985 & Expectation of Rand Index (RI)&~\citep{Hubert85} \\
\hline
2006 & Expectation and variance of $S \in \mathcal{L}$ &~\citep{Albatineh2006} \\
\hline
2009 & Expectation of Shannon Mutual Information (MI)&~\citep{Nguyen2009} \\
\hline
2010 & Expectation of Normalized Shannon MI (NMI) &~\citep{Nguyen2010} \\
\hline
2014 & Variance of Shannon MI &~\citep{Romano2014} \\
\hline
2015 & \specialcell{Expectation and variance of $S \in \mathcal{L}_\phi$\\Asymptotic expectation of $S \in \mathcal{N}_\phi$}& This Work \\
\hline
\end{tabular}
\caption{Survey of analytical computation of measures.} \label{tbl:contribution}
\end{table}

\section{Comparing Partitions} \label{sec:comppart}

Given two partitions (clusterings) $U$ and $V$ of the same data set of $N$ objects, let $\{ u_1, \dots, u_r\}$ and $\{ v_1,\dots,v_c\}$ be the disjoint sets (clusters) for $U$ and $V$ respectively. Let $|u_i| = a_i$ for $i=1,\dots,r$ denote the number of objects in the set $u_i$ and $|v_j| = b_j$ for $j=1,\dots,c$ denote the number of objects in $v_j$. Naturally, $\sum_{i=1}^r a_i = \sum_{j=1}^c b_j = N$. The overlap between the two partitions $U$ and $V$ can be represented in matrix form by a $r \times c$ contingency table $\mathcal{M}$ where $n_{ij}$ represents the number of objects in both $u_i$ and $v_j$, i.e.\ $n_{ij} = |u_i \cap v_j|$. Also, we refer to $a_i = \sum_{i=1}^r n_{ij}$ as the row marginals and to $b_j = \sum_{j=1}^c n_{ij}$ as the column marginals. A contingency table $\mathcal{M}$ is shown in Table~\ref{tbl:contingency}.
\begin{table}[h]
\centering
\begin{tabular}{c|c|ccccc|}
\multicolumn{2}{c}{} & \multicolumn{5}{c}{  $V$     }\\
\cline{3-7}
\multicolumn{2}{c|}{ } & $b_1$ & $\cdots$ & $b_j$ & $\cdots$ & $b_c$ \\
\cline{2-7}
\multirow{2}{*}{     }  & $a_1$ &
$n_{11}$ & $\cdots$ & $\cdot$ & $\cdots$ & $n_{1c}$ \\
& $\vdots$ &
$\vdots$ &  & $\vdots$ &  & $\vdots$ \\
 $U$  & $a_i$ &
$\cdot$ &  & $n_{ij}$ & & $\cdot$ \\
& $\vdots$ &
$\vdots$ & & $\vdots$ & & $\vdots$ \\
& $a_r$ &
$n_{r1}$ & $\cdots$ & $\cdot$ & $\cdots$ & $n_{rc}$ \\
\cline{2-7}
\end{tabular}
\caption{$r \times c$ contingency table $\mathcal{M}$ related to two clusterings $U$ and $V$. $a_i = \sum_j n_{ij}$ are the row marginals and $b_j = \sum_i n_{ij}$ are the column marginals.}\label{tbl:contingency}
\end{table}

\noindent Pair-counting measures between partitions, such as the Rand Index (RI)~\citep{Rand1971}, might be defined using the following quantities: $k_{11}$, the pairs of objects in the same set in both $U$ and $V$; $k_{00}$ the pairs of objects not in the same set in $U$ and not in the same set in $V$; $k_{10}$, the pairs of objects in the same set in $U$ and not in the same set in $V$; and $k_{01}$ the pairs of objects not in the same set in $U$ and in the same set in $V$. All these quantities can be computed using the contingency table $\mathcal{M}$, for example:
\begin{equation}
k_{11} = \frac{1}{2}\sum_{i=1}^r \sum_{j=1}^c n_{ij}(n_{ij} -1), \quad k_{00} = \frac{1}{2}\Big( N^2 + \sum_{i=1}^r \sum_{j=1}^c n_{ij}^2 - \Big( \sum_{i=1}^r a_i^2 + \sum_{j=1}^c b_j^2 \Big) \Big)
\end{equation}
Using $k_{00}$, $k_{11}$, $k_{10}$, and $k_{01}$ it is possible to compute similarity measures, e.g.\ RI, or distance measures, e.g.\ the Mirkin index $\mbox{MK}(U,V) \triangleq \sum_{i} a_i^2 + \sum_j b_j^2 - 2\sum_{i,j} n_{ij}^2$, between partitions~\citep{Meila2007}:
\begin{equation}
\mbox{RI}(U,V) \triangleq (k_{11} + k_{00})/\binom{N}{2}, \quad \mbox{MK}(U,V) = 2(k_{10} + k_{01}) = N(N-1)(1 - \mbox{RI}(U,V))
\end{equation}
Information theoretic measures are instead defined for random variables but can also be used to compare partitions when the we employ the empirical probability distributions associated to $U$, $V$, and the joint partition $(U,V)$. Let $\frac{a_i}{N}$, $\frac{b_j}{N}$, and $\frac{n_{ij}}{N}$ be the probability that an object falls in the set $u_i$, $v_j$, and $u_i \cap v_j$ respectively. We can therefore define the Shannon entropy with natural logarithms for a partition $V$ as follows: $H(V) \triangleq -\sum_j \frac{b_j}{N}\ln{\frac{b_j}{N}}$. Similarly, we can define the entropy $H(U)$ for the partition $U$, the joint entropy $H(U,V)$ for the joint partition $(U,V)$, and the conditional entropies $H(U|V)$ and $H(V|U)$. Shannon entropy can be used to define the well know mutual information (MI) and employ it to compute similarity between partitions $U$ and $V$:
\begin{equation}
\mbox{MI}(U,V) \triangleq H(U) - H(U|V) = H(V) - H(V|U) = H(U) + H(V) - H(U,V)
\end{equation}
On contingency tables, MI is linearly related to $G$-statistics used for likelihood-ratio tests: $ G = 2N\mbox{MI}$. In~\cite{Meila2007}, using the Shannon entropy it was shown that the following distance, namely the variation of information (VI) is a metric:
\begin{equation} \label{eq:vi}
\mbox{VI}(U,V) \triangleq 2H(U,V) -H(U) - H(V) = H(U|V) + H(V|U) = H(U) + H(V) -2\mbox{MI}(U,V)
\end{equation}
 
\subsection{Generalized Information Theoretic Measures} \label{sec:genent}

Generalized Information Theoretic (IT) measures based on the generalized Tsallis $q$-entropy \citep{Tsallis1988} can be defined for random variables~\citep{Furiuchi2006} and also be applied to the task of comparing partitions~\citep{Simovici2007}. Indeed, these measures have also seen recent application in the machine learning community. More specifically, it has been shown that they can act as proper kernels~\citep{Martins2009}. Furthermore, empirical studies demonstrated that careful choice of $q$ yields successful results when comparing the similarity between documents~\citep{Vila2011}, decision tree induction~\citep{Maszczyk2008}, and reverse engineering of biological networks~\citep{Lopes2011}. It is important to note that the Tsallis $q$-entropy is equivalent to the Harvda-Charvat-Dar\'{o}czy generalized entropy proposed in~\cite{Havrda1967, Daroczy1970}. Results available in literature about these generalized entropies are equivalently valid for all the proposed versions.

Given  $q \in \mathbb{R}^+ - \{1\}$, the generalized Tsallis $q$-entropy for a partition $V$ is defined as follows: $H_q(V) \triangleq \frac{1}{q - 1}\big( 1 - \sum_i \big( \frac{b_j}{N}\big)^q \big)$. Similarly to the case of Shannon entropy, we have the joint $q$-entropy $H_q (U,V)$ and the conditional $q$-entropies $H_q(U|V)$ and $H_q(V|U)$. Conditional $q$-entropy is computed according to a weighted average parametrized in $q$. More specifically the formula for $H_q(V|U)$ is:
\begin{equation} \label{eq:hugivenv}
H_q(V|U) \triangleq \sum_{i=1}^r \Big( \frac{a_i}{N} \Big)^q H_q(V|u_i) = \sum_{i=1}^r \Big( \frac{a_i}{N} \Big)^q \frac{1}{q - 1}\Big( 1 - \sum_{j=1}^c \Big( \frac{n_{ij}}{a_i}\Big)^q \Big)
\end{equation}
The $q$-entropy reduces to the Shannon entropy computed in nats for $q \rightarrow 1$.

In~\cite{Furiuchi2006}, using the fact that $q > 1$ implies $H_q(U) \geq H_q(U|V)$, it is shown that non-negative MI can be naturally generalized with $q$-entropy when $q>1$: 
\begin{equation} \label{eq:mi_beta}
\mbox{MI}_q(U,V) \triangleq H_q(U) - H_q(U|V) = H_q(V) - H_q(V|U) = H_q(U) + H_q(V) - H_q(U,V)
\end{equation}
However, $q$ values smaller than 1 are allowed if the assumption that MI$_q(U,V)$ is always positive can be dropped. In addition, generalized entropic measures can be used to defined the generalized variation of information distance ($\mbox{VI}_q$) which tends to VI in Eq. \eqref{eq:vi} when $q \rightarrow 1$:
\begin{equation} \label{eq:vi_beta}
\mbox{VI}_q(U,V) \triangleq H_q(U|V) + H_q(V|U) = 2H_q(U,V) -H_q(U) - H_q(V) = H_q(U) + H_q(V) -2\mbox{MI}_q(U,V)
\end{equation}
In~\cite{Simovici2007} it was shown that $\mbox{VI}_q$ is a proper metric and interesting links were identified between measures for comparing partitions $U$ and $V$. We state these links in Proposition \ref{prop:vi_ri} given that they set the fundamental motivation of our paper:
\begin{restatable}{prop}{propviri}{\citep{Simovici2007}}\label{prop:vi_ri}
When $q = 2$ the generalized variation of information, the Mirkin index, and the Rand index are linearly related: $\mbox{\emph{VI}}_2(U,V) = \frac{1}{N^2}\mbox{\emph{MK}}(U,V) = \frac{N-1}{N}(1 - \mbox{\emph{RI}}(U,V))$.
\end{restatable}
\noindent Generalized IT measures are not only a generalization of IT measures in the Shannon sense but also a generalization of pair-counting measures for particular values of $q$.

\subsection{Normalized Generalized IT Measures} \label{sec:normmi}

To allow a more interpretable range of variation, a clustering similarity measure should be normalized: it should achieve its maximum at 1 when $U = V$. An upper bound to the generalized mutual information MI$_q$ is used to obtained a normalized measure. $\mbox{MI}_q$ can take different possible upper bounds~\citep{Furiuchi2006}. Here, we choose to derive another possible upper bound using Eq.\ \eqref{eq:vi_beta} when we use the minimum value of $\mbox{VI}_q = 0$: $\max{ \mbox{MI}_q} = \frac{1}{2}( H_q(U) + H_q(V))$. This upper bound is valid for any $q \in \mathbb{R}^+ -\{1\}$ and allows us to link different existing measures as we will show in the next sections of the paper. The Normalized Mutual Information with $q$-entropy (NMI$_q$) is defined as follows:
\begin{equation} \label{eq:nmiq}
\mbox{NMI}_q(U,V) \triangleq \frac{\textup{MI}_q(U,V)}{\max{\textup{MI}_q(U,V) } } = \frac{\textup{MI}_q(U,V)}{\frac{1}{2} \big( H_q(U) + H_q(V) \big) } = \frac{H_q(U) + H_q(V) - H_q(U,V)}{\frac{1}{2} \big( H_q(U) + H_q(V) \big) }
\end{equation}
Even if NMI$_q(U,V)$ achieves its maximum 1 when the partitions $U$ and $V$ are identical, NMI$_q(U,V)$ is not a suitable clustering comparison measure. Indeed, it does not show constant baseline value equal to 0 when partitions are random. We explore this through an experiment. Given a dataset of $N = 100$ objects, we randomly generate uniform partitions $U$ with $r=2,4,6,8,10$ sets and $V$ with $c = 6$ sets \emph{independently} from each others. The average value of $\mbox{NMI}_q$ over $1,000$ simulations for different values of $q$ is shown in Figure \ref{fig:varK}. It is reasonable to expect that when the partitions are independent, the average value of $\mbox{NMI}_q$ is constant irrespectively of the number of sets $r$ of the partition $U$. This is not the case. This behaviour is unintuitive and misleading when comparing partitions.
\begin{figure}[h]
\centering
\includegraphics[scale=.7]{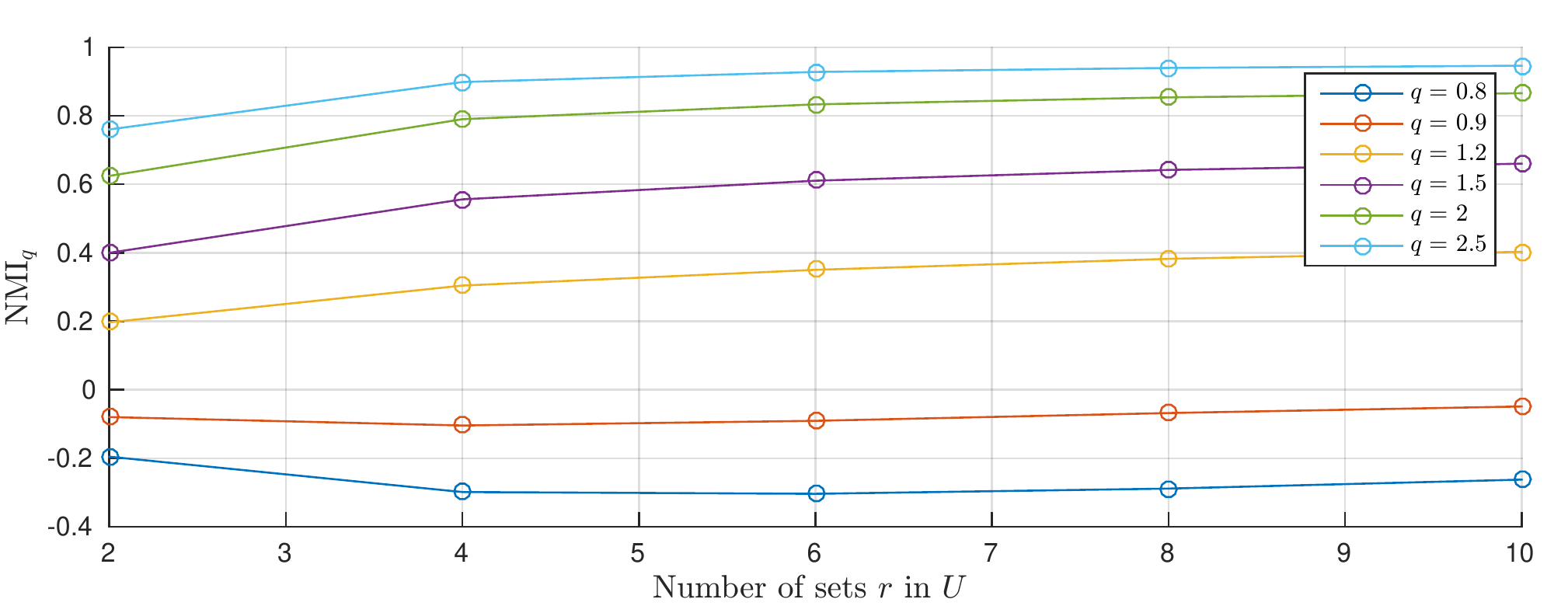}
\caption{The baseline value of $\mbox{NMI}_q$ between independent partitions is not constant.}\label{fig:varK}
\end{figure}
Computing the analytical expected value of generalized IT measures under the null hypothesis of random and independent $U$ and $V$ is important; it can be subtracted from the measure itself to adjust its baseline for chance such that their value is 0 when $U$ and $V$ are random. Given Proposition \ref{prop:vi_ri}, this strategy also allows us to generalize adjusted for chance pair-counting and Shannon IT measures. 

\section{Baseline Adjustment} 

In order to adjust the baseline of a similarity measure $S(U,V)$, we have to compute its expected value under the null hypothesis of independent partitions $U$ and $V$. We adopt the assumption used for RI~\citep{Hubert85} and the Shannon MI~\citep{Nguyen2009}: partitions $U$ and $V$ are generated independently with fixed number of points $N$ and fixed marginals $a_i$ and $b_j$; this is also denoted as the permutation or the hypergeometric model of randomness. We are able to compute the exact expected value for a similarity measure in the family $\mathcal{L}_\phi$:

\begin{defi}
Let $\mathcal{L}_\phi$ be the family of similarity measures $S(U,V) = \alpha + \beta \sum_{ij} \phi_{ij}(n_{ij})$ where $\alpha$ and $\beta$ do not depend on the entries $n_{ij}$ of the contingency table $\mathcal{M}$ and $\phi_{ij}(\cdot)$ are bounded real functions.
\end{defi}

\noindent Intuitively, $\mathcal{L}_\phi$ represents the class of measures that can be written as a linear combination of $\phi_{ij}(n_{ij})$. A measure between partitions uniquely determines $\alpha$, $\beta$, and $\phi_{ij}$. However, not every choice of $\alpha$, $\beta$, and $\phi_{ij}$ yields a meaningful similarity measure. $\mathcal{L}_\phi$ is a superset of the set of $\mathcal{L}$ defined in~\cite{Albatineh2006} as the family of measures $S(U,V) = \alpha + \beta \sum_{ij} n_{ij}^2$, i.e.\ $S \in \mathcal{L}$ are special cases of measures in $\mathcal{L}_\phi$ with $\phi_{ij}(\cdot) = (\cdot)^2$. 
Figure \ref{fig:families} shows a diagram of the similarity measures discussed in Section \ref{sec:genent} and their relationships.
\begin{restatable}{lem}{lememalpha} \label{lemma:emalpha}
If $S(U,V) \in \mathcal{L}_\phi$, when partitions $U$ and $V$ are random:
\begin{equation}
E[S(U,V)] = \alpha  + \beta \sum_{ij} E[\phi_{ij}(n_{ij})] \quad \text{ where } \quad E[\phi_{ij}(n_{ij})] \quad \text{ is } \label{eq:emu}\
\end{equation}
\begin{equation}
\sum_{n_{ij} = \max \{0,a_i+ b_j-N \} }^{\min \{ a_i, b_j \} } \phi_{ij}(n_{ij})
\frac{a_i!b_j!(N-a_i)!(N-b_j)!}{N!n_{ij}!(a_i - n_{ij})!(b_j - n_{ij})!(N - a_i - b_j + n_{ij})!} \label{eq:ephinij}
\end{equation}
\end{restatable}
\noindent Lemma \ref{lemma:emalpha} extends the results in~\cite{Albatineh2011} showing exact computation of the expected value of measures in the family $\mathcal{L}$. Given that generalized IT measures belong in $\mathcal{L}_\phi$ we can employ this result to adjust them.

\subsection{Baseline Adjustment for Generalized IT measures}

Using Lemma \ref{lemma:emalpha} it is possible to compute the exact expected value of $H_q(U,V)$, $\mbox{VI}_q(U,V)$ and $\mbox{MI}_q(U,V)$:
\begin{restatable}{thm}{thmvibeta} \label{thm:expvi}
When the partitions $U$ and $V$ are random:
\begin{enumerate}[topsep=0pt,itemsep=0ex,partopsep=1ex,parsep=1ex,label=\roman*)]
\item
$E[H_q(U,V)] = \frac{1}{q-1}\Big( 1 - \frac{1}{N^q}\sum_{ij} E[n_{ij}^q]\Big)$ with $E[n_{ij}^q]$ from Eq.\ \eqref{eq:ephinij} with $\phi_{ij}(n_{ij}) = n_{ij}^q$;
\item $E[\mbox{\emph{MI}}_q(U,V)] = H_q(U) + H_q(V) - E[H_q(U,V)]$;
\item $E[\mbox{\emph{VI}}_q(U,V)] = 2E[H_q(U,V)] - H_q(U) - H_q(V)$.
\end{enumerate}
\end{restatable}
\noindent It is worth noting that this approach is valid for any $q \in \mathbb{R}^+ - \{1\}$. We can use these expected values to adjust for baseline generalized IT measures. We use the method proposed in~\cite{Hubert85} to adjust similarity measures, such as $\mbox{MI}_q$, and distance measures, such as $\mbox{VI}_q$: \begin{equation} \label{eq:adj}
\mbox{AMI}_q  \triangleq \frac{\mbox{MI}_q - E[\mbox{MI}_q]}{\max{\mbox{MI}_q} - E[\mbox{MI}_q]} \quad
\mbox{AVI}_q  \triangleq \frac{E[\mbox{VI}_q] - \mbox{VI}_q}{E[\mbox{VI}_q] - \min{\mbox{VI}_q}} 
\end{equation}
$\mbox{VI}_q$ is a distance measure, thus $\min{\mbox{VI}_q} = 0$. For MI$_q$ we use the upper bound $\max{ \mbox{MI}_q} = \frac{1}{2}\big( H_q(U) + H_q(V)\big)$ as for NMI$_q$ in Eq.\ \eqref{eq:nmiq}. An exhaustive list of adjusted versions of Shannon MI can be found in~\cite{Nguyen2010}, when the upper bound $\frac{1}{2}( H_q(U) + H_q(V))$ is used the authors named the adjusted MI as $\mbox{AMI}_\textup{sum}$.

It is important to note that this type of adjustment turns distance measures into similarity measures, i.e.,\ $\mbox{AVI}_q$ is a similarity measure. It is also possible to maintain both the distance properties and the baseline adjustment using $\mbox{NVI}_q \triangleq \mbox{VI}_q/ E[\mbox{VI}_q]$ which can be seen as a normalization of $\mbox{VI}_q$ with the stochastic upper bound $E[\mbox{VI}_q]$~\citep{Nguyen2009}. It is also easy to see that $\mbox{AVI}_q = 1 - \mbox{NVI}_q$. The adjustments in Eq.\ \eqref{eq:adj} also enable the measures to be normalized. $\mbox{AMI}_q$ and $\mbox{AVI}_q$ achieve their maximum at 1 when $U = V$ and their minimum is 0 when $U$ and $V$ are random partitions.

According to the chosen upper bound for MI$_q$, we obtain the nice analytical form shown in Theorem \ref{thm:adjan}. Our adjusted measures quantify the discrepancy between the values of the actual contingency
table and their expected value in relation to the maximum discrepancy possible, i.e.\ the denominator
in Eq.\ \eqref{eq:adjan_eq}. It is also easy to see that all measures in $\mathcal{L}_\phi$ resemble this form when adjusted.
\begin{restatable}{thm}{adjan} \label{thm:adjan}
Using $E[n_{ij}^q]$ in Eq.\ \eqref{eq:ephinij} with $\phi_{ij}(n_{ij}) = n_{ij}^q$, the adjustments for chance for $\mbox{\emph{MI}}_q(U,V)$ and $\mbox{\emph{VI}}_q(U,V)$ are:
\begin{equation} \label{eq:adjan_eq}
\mbox{\emph{AMI}}_q(U,V) = \mbox{\emph{AVI}}_q(U,V) =\frac{\sum_{ij} n_{ij}^q - \sum_{ij} E[n_{ij}^q]}{ \frac{1}{2}\Big( \sum_i a_i^q + \sum_j b_j^q \Big) - \sum_{ij} E[n_{ij}^q]}
\end{equation}
\end{restatable}
\noindent From now on we only discuss AMI$_q$, given that it is identical to AVI$_q$.
There are notable special cases for our proposed adjusted generalized IT measures. In particular, the Adjusted Rand Index (ARI)~\citep{Hubert85} is just equal to AMI$_2$. ARI is a classic measure, heavily used for validation in social sciences and the most popular clustering validity index.
\begin{restatable}{cor}{adjspec}\label{cor:adjspec}
It holds true that:\\
 i) $\lim_{q \rightarrow 1} \mbox{\emph{AMI}}_q = \lim_{q \rightarrow 1} \mbox{\emph{AVI}}_q = \mbox{\emph{AMI}} = \mbox{\emph{AVI}}$ with Shannon entropy;\\
 ii) $\mbox{\emph{AMI}}_2 = \mbox{\emph{AVI}}_2 = \mbox{\emph{ARI}}$.
\end{restatable}

\noindent Therefore, using the permutation model we can perform baseline adjustment to generalized IT measures. Our generalized adjusted IT measures are a further generalization of particular well known adjusted measures such as AMI and ARI. It is worth noting, that ARI is equivalent to other well known measures for comparing partitions~\citep{Albatineh2006}. Furthermore, there is also a strong connection between ARI and Cohen's $\kappa$ statistics used to quantify inter-rater agreement~\citep{Warrens2008}.

\paragraph{Computational complexity:} The computational complexity of $\mbox{AMI}_q$ in Eq.\ \eqref{eq:adjan_eq} is dominated by the computation of the sum of the expected value of each cell.
\begin{restatable}{prop}{complexp}
The computational complexity of $\mbox{\emph{AMI}}_q$ is $O(N \cdot \max{ \{ r,c\}})$. 
\end{restatable}

\noindent If all the possible contingency tables $\mathcal{M}$ obtained by permutations were generated, the computational complexity of the exact expected value would be $O(N!)$. However, this can be dramatically reduced using properties of the expected value.

\subsection{Experiments on Measure Baseline}

Here we show that our adjusted generalized IT measures have a baseline value of 0 when comparing random partitions $U$ and $V$. In Figure \ref{fig:varKadj} we show the behaviour of $\mbox{AMI}_q$, ARI, and AMI on the same experiment proposed in Section \ref{sec:normmi}. They are all close to 0 with negligible variation when the partitions are random and independent. Moreover, it is interesting to see the equivalence of $\mbox{AMI}_2 $ and ARI. On the other hand, the equivalence of $\mbox{AMI}_q$ and AMI with Shannon's entropy is obtained only at the limit $q \rightarrow 1$.
\begin{figure}[h]
\centering
\begin{minipage}{0.49\textwidth}
\includegraphics[scale=.65]{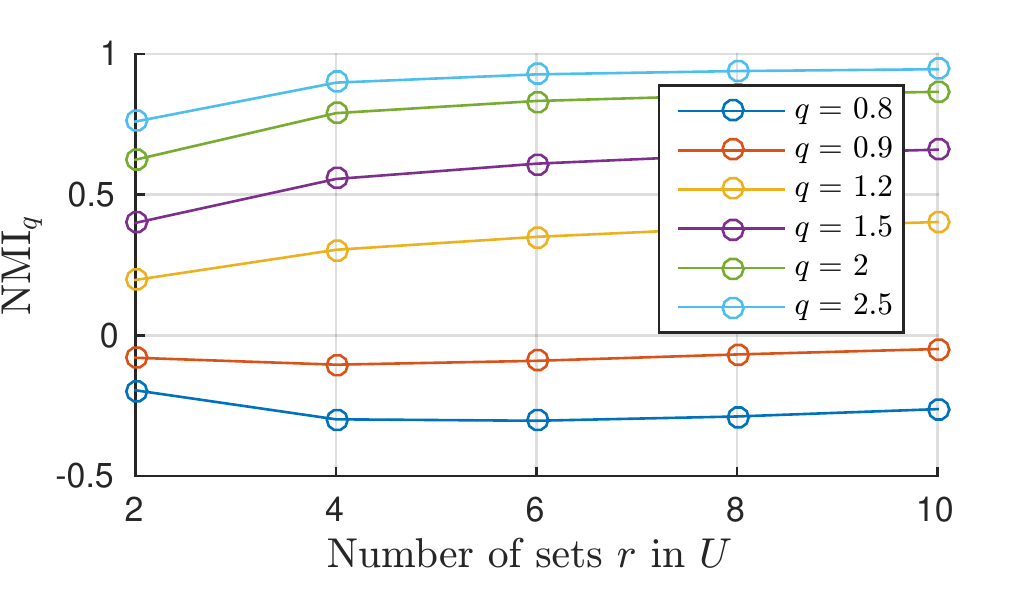}
\end{minipage}
\begin{minipage}{0.43\textwidth}
\includegraphics[scale=.65]{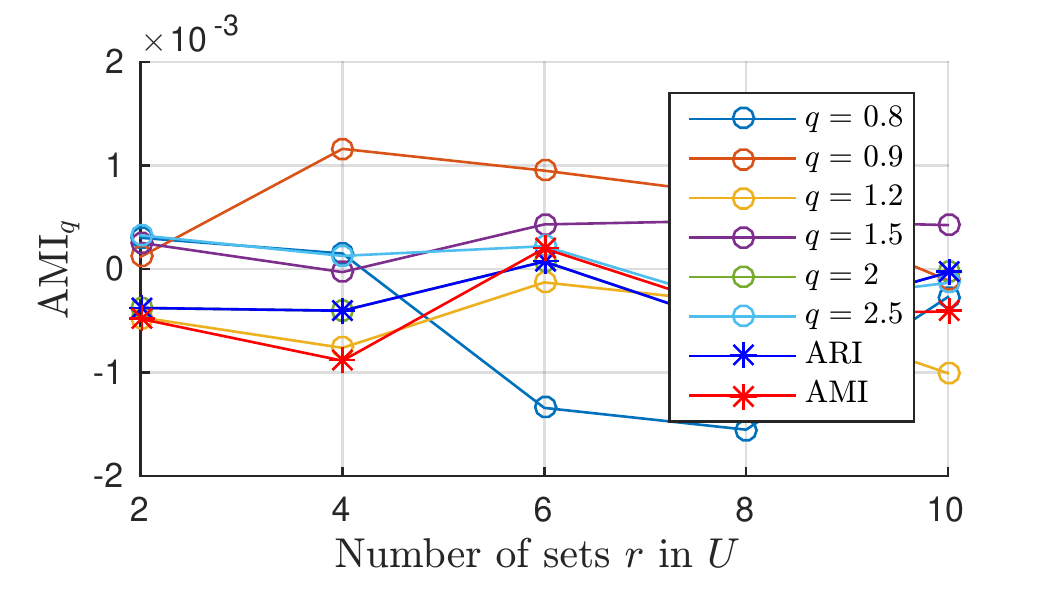}
\end{minipage}
\caption{When varying the number of sets for the random partition $U$, the value of $\mbox{AMI}_q(U,V)$ is always very close to 0 with negligible variation for any $q$.}\label{fig:varKadj}
\end{figure}

We also point out that $\mbox{NMI}_q$ does not show constant baseline when the relative size of the sets in $U$ varies when $U$ and $V$ are random. In Figure \ref{fig:varPercadj}, we generate random partitions $V$ with $c = 6$ sets on $N = 100$ points, and random binary partitions $U$ \emph{independently}. $\mbox{NMI}_q(U,V)$ shows different behavior at the variation of the relative size of the biggest set in $U$. This is unintuitive given that the partitions $U$ and $V$ are random and independent. We obtain the desired property of a baseline value of 0 with $\mbox{AMI}_q$.
\begin{figure}[h]
\centering
\begin{minipage}{0.49\textwidth}
\includegraphics[scale=.65]{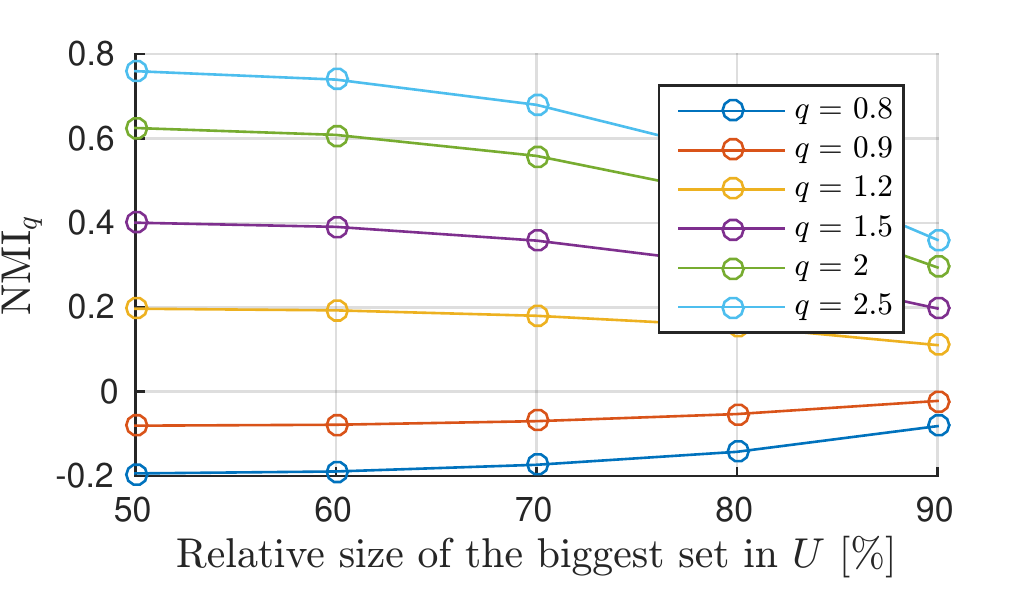}
\end{minipage}
\begin{minipage}{0.43\textwidth}
\includegraphics[scale=.65]{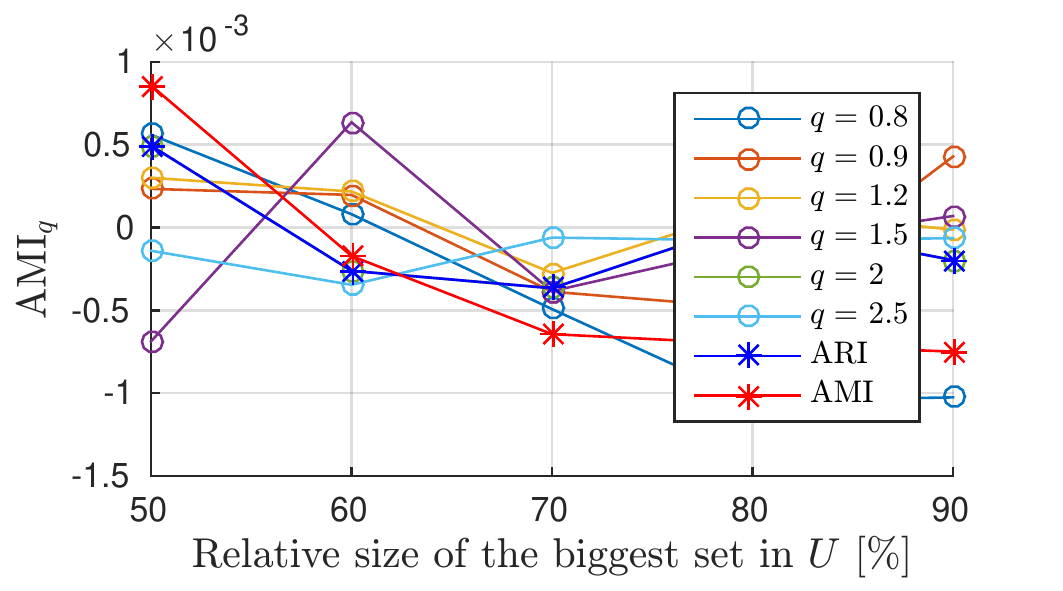}
\end{minipage}
\caption{When varying the relative size of one cluster for the random partition $U$, the value of $\mbox{AMI}_q(U,V)$ is always very close to 0 with negligible variation for any $q$.}\label{fig:varPercadj}
\end{figure}

\subsection{Large Number of Objects}

In this section, we introduce a very general family of measures which includes $\mathcal{L}_\phi$. For measures belonging to this family, it is possible to find an approximation of their expected value when the number of objects $N$ is large. This allows us to identify approximations for the expected value of measures in $\mathcal{L}_\phi$ as well as for measures not in $\mathcal{L}_\phi$, such as the Jaccard coefficient as shown in Figure \ref{fig:families}.

Let $\mathcal{N}_\phi$ be the family of measures which are \emph{non}-linear combinations of $\phi_{ij}(n_{ij})$:

\begin{defi}
Let $\mathcal{N}_\phi$ be the family of similarity measures $S(U,V) = \phi(\frac{n_{11}}{N},\dots,\frac{n_{ij}}{N},\dots,\frac{n_{rc}}{N})$ where $\phi$ is a bounded real function as $N$ reaches infinity.
\end{defi}
Note that $\mathcal{N}_\phi$ is a generalization of $\mathcal{L}_\phi$.
At the limit of large number of objects $N$, it is possible to compute the expected value of measures in $\mathcal{N}_\phi$ under random partitions $U$ and $V$ using only the marginals of the contingency table $\mathcal{M}$:
\begin{restatable}{lem}{largephi} \label{lemma:largephi}
If $S(U,V) \in \mathcal{N}_\phi$, then $\lim_{N\rightarrow +\infty} E[S(U,V)]=  \phi\Big( \frac{a_1}{N}\frac{b_1}{N},\dots,\frac{a_i}{N}\frac{b_j}{N},\dots,\frac{a_r}{N}\frac{b_c}{N}\Big)$.
\end{restatable}
\noindent In~\cite{Agresti1984} the expected value of the RI was computed using an approximated value based on the multinomial distribution. It turns out this approximated value is equal to what we obtain for RI using Lemma \ref{lemma:largephi}. The authors of~\citep{Albatineh2006} noticed that the difference between the approximation and the expected value obtained with the hypergeometric model is small on empirical experiments when $N$ is large. We point out that this is a natural consequence of Lemma \ref{lemma:largephi} given that $\mbox{RI} \in \mathcal{L}_\phi	\subseteq \mathcal{N}_\phi$.  Moreover, the multinomial distribution was also used to compute the expected value of the Jaccard coefficient (J) in~\cite{Albatineh2011}, obtaining good results on empirical experiments with many objects. Again, this is a natural consequence of Lemma \ref{lemma:largephi} given that $\mbox{J} \in \mathcal{N}_\phi$ but $\mbox{J} \notin \mathcal{L}_\phi$. 

Generalized IT measures belong in $\mathcal{L}_\phi	\subseteq \mathcal{N}_\phi$. Therefore we can employ Lemma \ref{lemma:largephi}. When the number of objects is large, the expected value under random partitions $U$ and $V$ of $H_q(U,V)$, MI$_q(U,V)$, and VI$_q(U,V)$ in Theorem \ref{thm:expvi} depends only on the entropy of the partitions $U$ and $V$:
\begin{restatable}{thm}{largeNmean} \label{thm:largeNmean} 
It holds true that:
\begin{enumerate}[topsep=0pt,itemsep=0ex,partopsep=1ex,parsep=1ex,label=\roman*)]
\item $\lim_{N \rightarrow +\infty} E[H_q(U,V)] = H_q(U) + H_q(V) - (q -1) H_q(U)H_q(V)$;
\item $\lim_{N \rightarrow +\infty} E[\mbox{\emph{MI}}_q(U,V)] = (q -1) H_q(U)H_q(V)$;
\item $\lim_{N \rightarrow +\infty} E[\mbox{\emph{VI}}_q(U,V)] = H_q(U) + H_q(V) - 2(q -1) H_q(U)H_q(V)$.
\end{enumerate}
\end{restatable}

\noindent Result i) recalls the property of non-additivity that holds true for random variables~\citep{Furiuchi2006}. Figure \ref{fig:largeNmean} shows the behaviour of $E[H_q(U,V)]$ when the partitions $U$ and $V$ are generated uniformly at random. $V$ has $c = 6$ sets and $U$ has $r$ sets. On this case, $H_q(U) + H_q(V) - (q -1) H_q(U)H_q(V)$ appears to be a good approximation already for $N = 1000$. In particular, the approximation is good when the number of records $N$ is big with regards to the number of cells of the contingency table in Table \ref{tbl:contingency}: i.e.,\ when $\frac{N}{r \cdot c}$ is large enough.
\begin{figure}[h]
\centering
\subfigure[Approximation on smaller sample size]{\includegraphics[scale=.65]{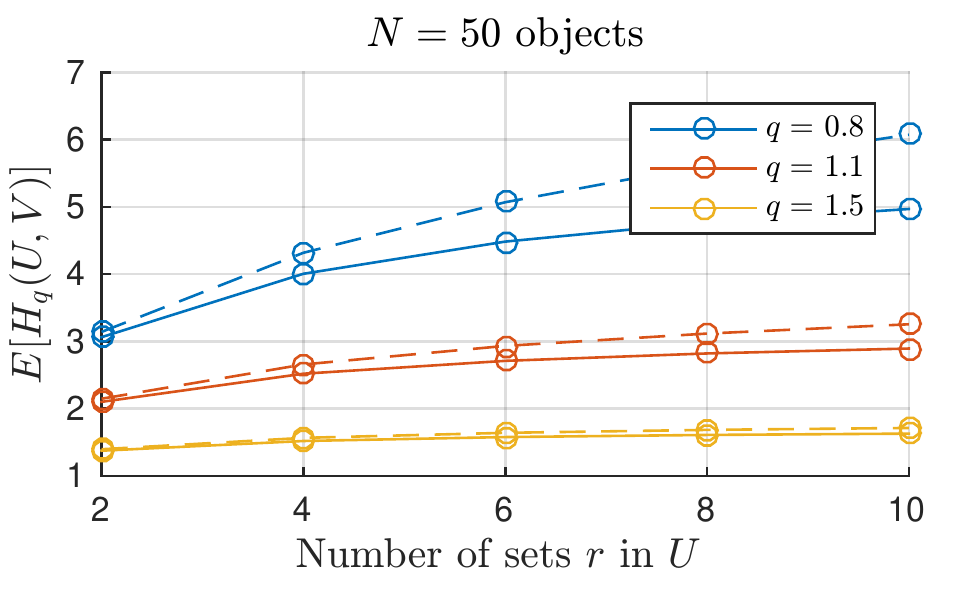} \label{fig:small}}
\quad \quad
\subfigure[Approximation on bigger sample size]{\includegraphics[scale=.65]{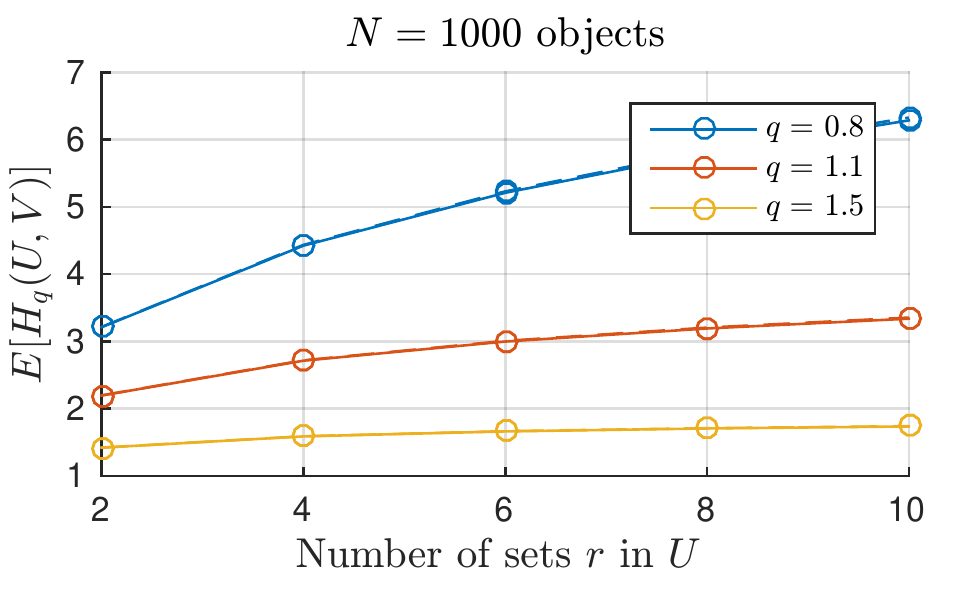} \label{fig:big}}
\caption{$E[H_q(U,V)]$ (solid) and their limit value $H_q(U) + H_q(V) - (q -1) H_q(U)H_q(V)$ (dashed). The solid line coincides approximately with the dashed one in \ref{fig:big} when $N = 1000$. The limit value is a good approximation for $E[H_q(U,V)]$ when $\frac{N}{r \cdot c} $ is large enough.} \label{fig:largeNmean}
\end{figure}

\section{Application scenarios for AMI$_q$}

In this section we aim to answer to the question: \emph{Given a reference ground truth clustering $V$, which is the best choice for $q$ in \emph{AMI}$_q(U,V)$ to validate the clustering solution $U$?} By answering this question, we implicitly identify the application scenarios for ARI and AMI given the results in Corollary~\ref{cor:adjspec}. This is particularly important for external clustering validation. Nonetheless, there are a number of other applications where the task is to find the most similar partition to a reference ground truth partition: e.g.,\ categorical feature selection~\citep{Vinh2014}, decision tree induction~\citep{Crimisini2012}, generation of alternative or multi-view clusterings~\citep{Muller2013}, or the exploration of the clustering space with the Meta-Clustering algorithm~\citep{Caruana2006,Lei2014} to list a few.

Different values for $q$ in AMI$_q$ yield to different biases. The source of these biases can be identified analyzing the properties of the $q$-entropy. In Figure~\ref{fig:qent} we show the $q$-entropy for a binary partition at the variation of the relative size $p$ of one cluster. This can be analytically computed: $H_q(p) = \frac{1}{q-1}(1 - p^q-(1-p)^q)$. \emph{The range of variation for $H_q(p)$ is much bigger if $q$ is small}. More specifically when $q$ is small, the difference in entropy between an unbalanced partition and a balanced partition is big. 
\begin{SCfigure}[2][h]
\centering
\caption{Tsallis $q$-entropy $H_q(p)$ for a binary clustering where $p$ is the relative size of one cluster. When $q$ is small, the $q$-entropy varies in bigger range.}
\includegraphics[scale=0.7]{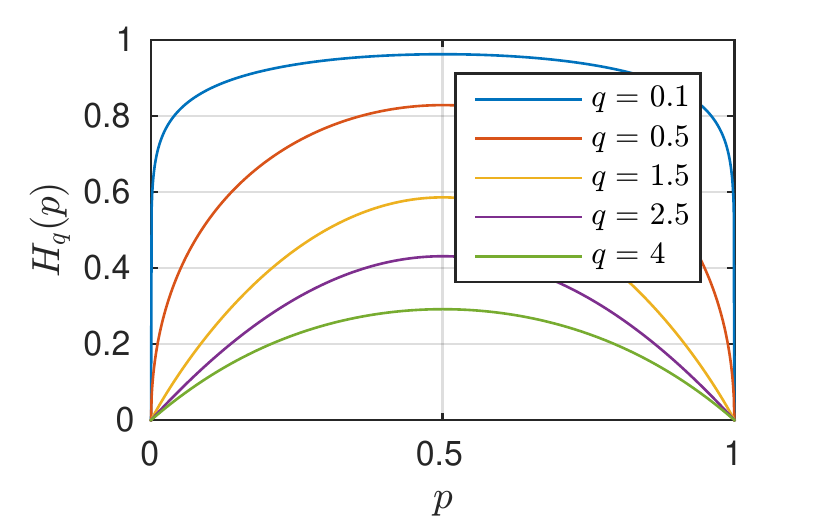}
\label{fig:qent}
\end{SCfigure}

\noindent Let us focus on an example. Let $V$ be a reference clustering with 3 clusters of size 50 each, and let $U_1$ and $U_2$ be two clustering solutions with the same number of clusters and same cluster sizes. The contingency tables for $U_1$ and $U_2$ are shown on Figure~\ref{fig:samemarginals}.
\begin{figure}[h]
\begin{minipage}{0.31\textwidth}
\begin{tabular}{c|c|ccc|}
\multicolumn{2}{c}{} & \multicolumn{3}{c}{  $V$     }\\
\cline{3-5}
\multicolumn{2}{c|}{ } & 50 & 50 & 50 \\
\cline{2-5}
\multirow{2}{*}{     }  & 50 &
50 & 0 & 0\\
$U_1$  & 50 &
0 & 44 & 6\\
 & 50 &
0 & 6 & 44\\
\cline{2-5}
\end{tabular}
\end{minipage}
\begin{minipage}{0.34\textwidth}
\begin{tabular}{c|c|ccc|}
\multicolumn{2}{c}{} & \multicolumn{3}{c}{  $V$     }\\
\cline{3-5}
\multicolumn{2}{c|}{ } & 50 & 50 & 50 \\
\cline{2-5}
\multirow{2}{*}{     }  & 50 &
48 & 1 & 1\\
$U_2$  & 50 &
1 & 46 & 3\\
 & 50 &
1 & 3 & 46\\
\cline{2-5}
\end{tabular}
\end{minipage}
\begin{minipage}{0.33\textwidth}
\includegraphics[scale=.7]{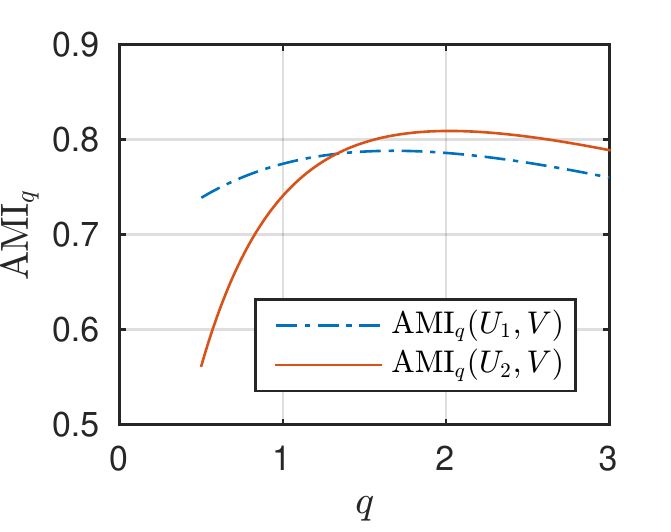}
\end{minipage}
\caption{AMI$_q$ with small $q$ prefers the solution $U_1$ because there exist one pure cluster: i.e.,\ there is one cluster which contains elements from only one cluster in the reference clustering $V$.}
\label{fig:samemarginals}
\end{figure}
Given that both contingency tables have the same marginals, the only difference between AMI$_q(U_1,V)$ and AMI$_q(U_2,V)$ according to Eq.\ \eqref{eq:adj} lies in MI$_q$. Given that both solutions $U_1$ and $U_2$ are compared against $V$, the only term that varies in $\mbox{MI}_q(U,V) = H_q(V) - H_q(V|U)$ is $H_q(V|U)$. In order to identify the clustering solution that maximizes AMI$_q$ we have to analyze the solution that decreases $H_q(V|U)$ the most. $H_q(V|U)$ is a weighted average of the entropies $H_q(V|u_i)$ computed on the rows of the contingency table as shown in Eq.\ \eqref{eq:hugivenv}, and this is sensitive to values equal to 0. Given the bigger range of variation of $H_q$ for small $q$, small $q$ implies higher sensitivity to row entropies of 0. Therefore, small values of $q$ tends to decrease $H_q(V|U)$ much more if the clusters in the solution $V$ are pure: i.e.,\ clusters contain elements from only one cluster in the reference clustering $V$. In other words, \emph{\emph{AMI}$_q$ with small $q$ prefers pure clusters in the clustering solution.}

When the marginals in the contingency tables for two solutions are different, another important factor in the computation of AMI$_q$ is the normalization coefficient $\frac{1}{2}(H_q(U) + H_q(V))$. Balanced solutions $V$ will be penalized more by AMI$_q$ when $q$ is small. Therefore, \emph{\emph{AMI}$_q$ with small $q$ prefers unbalanced clustering solutions}. To summarize, AMI$_q$ with small $q$ such as AMI$_{0.5}$ or $\mbox{AMI}_1 =\mbox{AMI}$ with Shannon's entropy:
\begin{itemize}[topsep=1ex,itemsep=0ex,partopsep=1ex,parsep=1ex]
\item Is biased towards pure clusters in the clustering solutions;
\item Prefers unbalanced clustering solutions.
\end{itemize}
By contrary, AMI$_q$ with bigger $q$ such as AMI$_{2.5}$ or $\mbox{AMI}_{2} =\mbox{ARI}$:
\begin{itemize}[topsep=1ex,itemsep=0ex,partopsep=1ex,parsep=1ex]
\item Is less biased towards pure clusters in the clustering solution;
\item Prefers balanced clustering solutions.
\end{itemize}
Given a reference clustering $V$, these biases can guide the choice of $q$ in AMI$_q$ to identify more suitable clustering solutions.

\subsection{Use AMI$_q$ with small $q$ such as AMI$_{0.5}$ or $\mbox{AMI}_1 =\mbox{AMI}$ when the reference clustering is unbalanced and there exist small clusters}

If the reference cluster $V$ is unbalanced and presents small clusters, AMI$_q$ with small $q$ might prefer more appropriate clustering solutions $U$. For example, in Figure~\ref{fig:betterAMI} we show two contingency tables associated to two clustering solutions $U_1$ and $U_2$ for the reference clustering $V$ with 4 clusters of size $[10, 10, 10, 70]$ respectively. When there exist small clusters in the reference $V$ their identification has to be \emph{precise} in the clustering solution. The solution $U_1$ looks arguably better than $U_2$ because it shows many pure clusters. In this scenario we advise the use of AMI$_{0.5}$ or $\mbox{AMI}_1 =\mbox{AMI}$ with Shannon's entropy because it gives more weight to the clustering solution $U_1$.
\begin{figure}[h]
\begin{minipage}{0.31\textwidth}
\small
\begin{tabular}{c|c|cccc|}
\multicolumn{2}{c}{} & \multicolumn{4}{c}{     $V$     }\\
\cline{3-6}
\multicolumn{2}{c|}{ } & $10$ & $10$ & $10$ & $70$ \\
\cline{2-6}
\multirow{4}{*}{    $U_1$       }  
& $8$ & $8$ & $0$ & $0$ & $0$ \\
& $7$ & $0$  & $7$ & $0$ & $0$ \\
& $7$ & $0$  & $0$ & $7$ & $0$ \\
& $78$ & $2$  & $3$ & $3$ & $70$ \\
\cline{2-6}
\end{tabular}
\end{minipage}
\begin{minipage}{0.34\textwidth}
\small
\begin{tabular}{c|c|cccc|}
\multicolumn{2}{c}{} & \multicolumn{4}{c}{     $V$     }\\
\cline{3-6}
\multicolumn{2}{c|}{ } & $10$ & $10$ & $10$ & $70$ \\
\cline{2-6}
\multirow{4}{*}{    $U_2$       }  
& $10$ & $7$ & $1$ & $1$ & $1$ \\
& $10$ & $1$  & $7$ & $1$ & $1$ \\
& $10$ & $1$  & $1$ & $7$ & $1$ \\
& $70$ & $1$  & $1$ & $1$ & $67$ \\
\cline{2-6}
\end{tabular}
\end{minipage}
\begin{minipage}{0.33\textwidth}
\includegraphics[scale=.7]{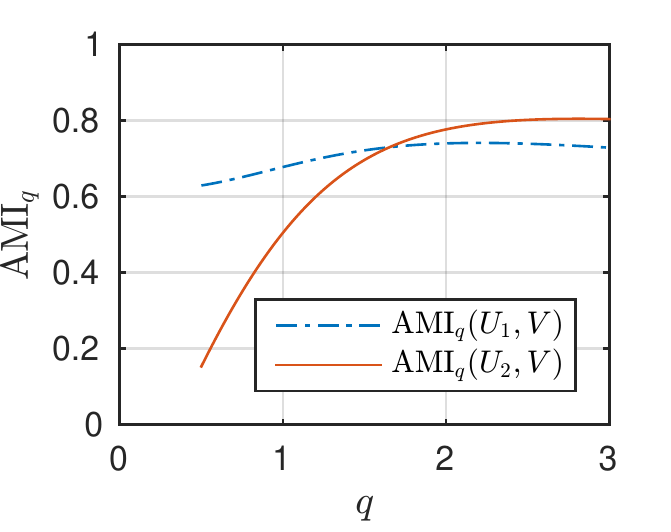}
\end{minipage}
\caption{AMI$_q$ with small $q$ prefers the solution $U_1$ because its clusters are pure. When the reference clustering has small clusters their identification in the solution has to be \emph{precise}. In this scenario we advise the use of AMI$_{0.5}$ or $\mbox{AMI}_1 =\mbox{AMI}$.}
\label{fig:betterAMI}
\end{figure}

\subsection{Use AMI$_q$ with big $q$ such as AMI$_{2.5}$ or $\mbox{AMI}_2 =\mbox{ARI}$ when the reference clustering has big equal sized clusters}

If $V$ is a reference clustering with big equal size clusters it is less crucial to have precise clusters in the solution. Indeed, precise clusters in the solution penalize the \emph{recall} of clusters in the reference. In this case, AMI$_q$ with bigger $q$ might prefer more appropriate solutions.
\begin{figure}[h]
\begin{minipage}{0.31\textwidth}
\small
\begin{tabular}{c|c|cccc|}
\multicolumn{2}{c}{} & \multicolumn{4}{c}{     $V$     }\\
\cline{3-6}
\multicolumn{2}{c|}{ } & $25$ & $25$ & $25$ & $25$ \\
\cline{2-6}
\multirow{4}{*}{    $U_1$       }  
& $17$ & $17$ & $0$ & $0$ & $0$ \\
& $17$ & $0$  & $17$ & $0$ & $0$ \\
& $17$ & $0$  & $0$ & $17$ & $0$ \\
& $49$ & $8$  & $8$ & $8$ & $25$ \\
\cline{2-6}
\end{tabular}
\end{minipage}
\begin{minipage}{0.34\textwidth}
\small
\begin{tabular}{c|c|cccc|}
\multicolumn{2}{c}{} & \multicolumn{4}{c}{     $V$     }\\
\cline{3-6}
\multicolumn{2}{c|}{ } & $25$ & $25$ & $25$ & $25$ \\
\cline{2-6}
\multirow{4}{*}{    $U_2$       }  
& $24$ & $20$ & $2$ & $1$ & $1$ \\
& $25$ & $2$  & $20$ & $2$ & $1$ \\
& $23$ & $1$  & $1$ & $20$ & $1$ \\
& $28$ & $2$  & $2$ & $2$ & $22$ \\
\cline{2-6}
\end{tabular}
\end{minipage}
\begin{minipage}{0.33\textwidth}
\includegraphics[scale=.7]{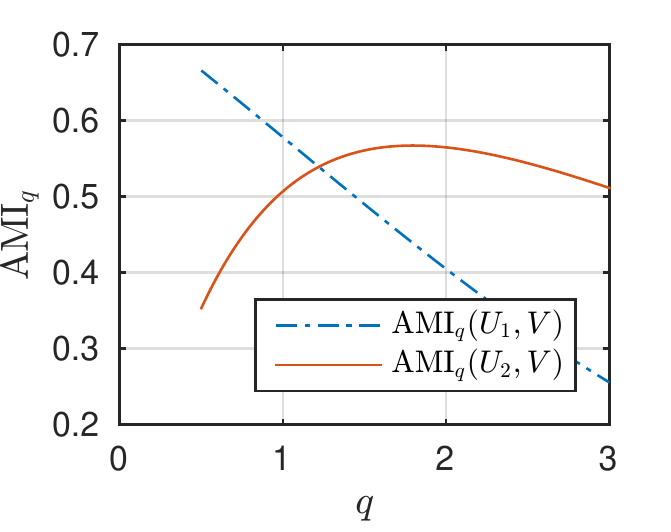}
\end{minipage}
\caption{AMI$_q$ with big $q$ prefers the solution $U_2$ because it is less biased to pure clusters in the solution. When the reference clustering has big equal sized clusters their precise identification is less crucial. In this scenario we advise the use of AMI$_{2.5}$ or $\mbox{AMI}_2 =\mbox{ARI}$.}
\label{fig:betterARI}
\end{figure}
In Figure~\ref{fig:betterARI} we show two clustering solutions $U_1$ and $U_2$ for the reference clustering $V$ with 4 equal size clusters of size 25. The solution $U_2$ looks better than $U_1$ because each of its clusters identify more elements from particular clusters in the reference. Moreover, $U_2$ has to be preferred to $U_1$ because it consists in 4 equal sized clusters like for the reference clustering $V$. In this scenario we advise the use of AMI$_{2.5}$ or $\mbox{AMI}_2 =\mbox{ARI}$ because it gives more importance to the solution $U_2$.

\section{Standardization of Clustering Comparison Measures}

Selection of the most similar partition $U$ to a reference partition $V$ is biased according to the chosen similarity measure, the number of sets $r$ in $U$, and their relative size. This phenomena is known as \emph{selection bias} and it has been extensively studied in decision trees~\citep{White1994}. Researchers in this area agree that in order to achieve unbiased selection of partitions, distribution properties of similarity measures have to be taken into account~\citep{Dobra01, Shih2004, Hothorn2006}.
Using the permutation model, we proposed in~\cite{Romano2014} to analytically standardize the Shannon MI by subtraction of its expected value and division by its standard deviation. In this Section, we discuss how to achieve analytical standardization of measures $S \in \mathcal{L}_\phi$.

In order to standardize measures $S(U,V)$ we must analytically compute their variance:
\begin{restatable}{lem}{lemanalalpha} \label{lem:analalpha}
If $S(U,V) \in \mathcal{L}_\phi$, when partitions $U$ and $V$ are random:
\begin{align*}
\mbox{\emph{Var}}(S(U,V)) =
\beta^2\Big( E\Big[ \Big(\sum_{ij} \phi_{ij}(n_{ij})\Big)^2 \Big] - \Big( \sum_{ij} E[\phi_{ij}(n_{ij})] \Big)^2 \Big) \text{where } E \Big[ \Big( \sum_{ij} \phi_{ij}(n_{ij}) \Big)^2\Big] \text{is }
\end{align*}
\begin{align} 
\sum_{ij} \sum_{ n_{ij}}   \phi(n_{ij}) P(n_{ij}) \cdot \Bigg[&  \phi_{ij}(n_{ij}) + \sum_{i'\neq i} \sum_{\tilde{n}_{i'j} }\phi_{i'j}(\tilde{n}_{i'j}) P(\tilde{n}_{i'j}) + \nonumber\\
& + \sum_{j'\neq j}  \sum_{\tilde{n}_{ij'} } P(\tilde{n}_{ij'}) \Bigg( \phi_{ij'}(\tilde{n}_{ij'}) + \sum_{i' \neq i} \sum_{\tilde{\tilde{n}}_{i'j'} }\phi_{i'j'}(\tilde{\tilde{n}}_{i'j'})P(\tilde{\tilde{n}}_{i'j'}) \Bigg) \Bigg] \label{eq:enij2alpha}
\end{align}
where $n_{ij} \sim \mbox{\emph{Hyp}}(a_i,b_j,N)$, $\tilde{n}_{i'j} \sim \mbox{\emph{Hyp}}(b_j - n_{ij}, a_{i'}, N - a_i)$, $\tilde{n}_{ij'} \sim \mbox{\emph{Hyp}}(a_i - n_{ij}, b_{j'}, N - b_j)$, $\tilde{\tilde{n}}_{i'j'} \sim \mbox{\emph{Hyp}}(a_{i'},b_{j'} - \tilde{n}_{ij'}, N - a_i)$ are hypergeometric random variables.
\end{restatable}
\noindent We can use the expected value to standardize $S \in \mathcal{L}_\phi$, such as generalized IT measures.

\subsection{Standardization of Generalized IT Measures}

The variance under the permutation model of generalized IT measures is:
\begin{restatable}{thm}{thmvarsvi} \label{thm:varsvi} Using Eqs. \eqref{eq:ephinij} and \eqref{eq:enij2alpha} with $\phi_{ij}(\cdot) = (\cdot)^q$, when the partitions $U$ and $V$ are random:
\begin{enumerate}[topsep=0pt,itemsep=0ex,partopsep=1ex,parsep=1ex,label=\roman*)]
\item 
$\mbox{\emph{Var}}(H_q(U,V)) = \frac{1}{(q -1)^2N^{2q}} \Big( E[ (\sum_{ij} n_{ij}^q)^2] - ( \sum_{ij} E[n_{ij}^q ] )^2 \Big)$;
\item $\mbox{\emph{Var}}(\mbox{\emph{MI}}_q(U,V)) = \mbox{\emph{Var}}(H_q(U,V))$ 
\item $\mbox{\emph{Var}}(\mbox{\emph{VI}}_q(U,V)) = 4\mbox{\emph{Var}}(H_q(U,V))$
\end{enumerate}
\end{restatable}
\noindent We define the standardized version of the similarity measure $\mbox{MI}_q$ (SMI$_q$), and the standardized version of the distance measure VI$_q$ (SVI$_q$) as follows:
\begin{equation}
\mbox{SMI}_q \triangleq \frac{\mbox{MI}_q - E[\mbox{MI}_q]}{\sqrt{\mbox{Var}(\mbox{MI}_q)}}, \quad \mbox{SVI}_q \triangleq \frac{E[\mbox{VI}]_q - \mbox{VI}_q}{\sqrt{\mbox{Var}(\mbox{VI}_q)}},
\end{equation} 
As for the case of AMI$_q$ and AVI$_q$, it turns out that SMI$_q$ is equal to SVI$_q$:
\begin{restatable}{thm}{thmsvianal}
Using Eqs. \eqref{eq:ephinij} and \eqref{eq:enij2alpha} with $\phi_{ij}(\cdot) = (\cdot)^q$, the standardized \emph{MI}$_q(U,V)$ and the standardized \emph{VI}$_q(U,V)$ are:
\begin{equation} \label{eq:smianal}
\mbox{\emph{SMI}}_q(U,V) = \mbox{\emph{SVI}}_q(U,V) = \frac{\sum_{ij}n_{ij}^q - \sum_{ij}E[n_{ij}^q]}{ \sqrt{ E[(\sum_{ij}n_{ij}^q)^2] - (\sum_{ij}E[n_{ij}^q])^2 }}
\end{equation}
\end{restatable}
\noindent This formula shows that we are interested in maximizing the difference between the sum of the cells of the actual contingency table and the sum of the expected cells under randomness. However, standardized measures differs from their adjusted counterpart because of the denominator, i.e.\ the standard deviation of the sums of the cells. Indeed, SMI$_q$ and SVI$_q$ measure the number of standard deviations MI$_q$ and VI$_q$ are from their mean.

There are some notable special cases for particular choices of $q$.
Indeed, our generalized standardization of IT measures allows us to generalize also the standardization of pair-counting measures such as the Rand Index. To see this, let us define the Standardized Rand Index (SRI): $\mbox{SRI} \triangleq \frac{\textup{RI} - E[\textup{RI}]}{\sqrt{\textup{Var}(\textup{RI})}}$ and recall that the standardized $G$-statistic is defined as $\mbox{S}G \triangleq \frac{G - E[G]}{\sqrt{\textup{Var}(G)}}$~\citep{Romano2014}:
\begin{restatable}{cor}{corspec} \label{cor:specsvi}
It holds true that:\\
i) $\lim_{q \rightarrow 1} \mbox{\emph{SMI}}_q = \lim_{q \rightarrow 1} \mbox{\emph{SVI}}_q = \mbox{\emph{SMI}} = \mbox{\emph{SVI}} = \mbox{\emph{S}}G$ with Shannon entropy; \\ 
ii) $\mbox{\emph{SMI}}_2 = \mbox{\emph{SVI}}_2 = \mbox{\emph{SRI}}$.
\end{restatable}

\paragraph{Computational complexity:} The computational complexity of $\mbox{SVI}_q$ is dominated by computation of the second moment of the sum of the cells defined in Eq.\ \eqref{eq:enij2alpha}:
\begin{restatable}{prop}{complvar}
The computational complexity of $\mbox{\emph{SVI}}_q$ is $O( N^3c \cdot \max{ \{c, r\} })$.
\end{restatable}
\noindent Note that the complexity is quadratic in $c$ and linear in $r$. This happens because of the way we decided to condition the probabilities in Eq.\ \eqref{eq:enij2alpha} in the proof of Lemma \ref{lem:analalpha}. With different conditions, it is possible to obtain a formula symmetric to Eq.\ \eqref{eq:enij2alpha} with complexity $O( N^3 r\cdot \max{ \{r, c\} })$ \citep{Romano2014}.

\paragraph{Statistical inference:} All IT measures computed on \emph{partitions} can be seen as estimators of their true value computed using the random \emph{variables} associated to the partitions $U$ and $V$. Therefore, SMI$_q$ can be seen as a non-parametric test for independence for MI$_q$. We formalize this with the following proposition:
\begin{restatable}{prop}{proppval}
The $p$-value associated to the test for independence between $U$ and $V$ using \emph{MI}$_q(U,V)$ is smaller than: $\frac{1}{1 + \left( \textup{SMI}_q(U,V) \right)^2}$.
\end{restatable}
\noindent For example, if SMI$_q$ is equal to 4.46 the associated $p$-value is smaller than 0.05. Neural time series data is often analyzed making use of the Shannon MI (e.g.\ see Chapter 29 in~\cite{Cohen2014}). It is common practice to test the independence of two time series by computing SMI via Monte Carlo permutations, sampling from the space of $N!$ cardinality. Our SMI$_q$ can be effectively and efficiently used in this application because it is exact and obtains $O( N^3 r\cdot \max{ \{r, c\} })$ complexity.

\subsection{Experiments on Selection Bias}

In this section, we evaluate the performance of standardized measures on selection bias correction when partitions $U$ are generated at random and independently from the reference partition $V$. This hypothesis has been employed in previous published research to study selection bias~\citep{White1994, Frank98, Dobra01, Shih2004, Hothorn2006,Romano2014}. In particular, we experimentally demonstrate that NMI$_q$ is biased towards the selection of partitions $U$ with more clusters at any $q$. Therefore, in this scenario it is beneficial to perform standardization. Although the choice of whether performing standardization or not is dependent to the application \citep{Romano2015framework}. For example, it has been argued that in some cases the selection of clustering solutions should be biased towards solutions with the same number of clusters as in the reference~\citep{Amelio2015}. In this section we aim to show the effects of selection bias when clusterings are independent and that standardization helps in reducing it. Moreover, we will see in Section \ref{sec:largeobjstd} that it is particularly important to correct for selection bias when the number of records $N$ is small.

Given a reference partition $V$ on $N = 100$ objects with $c = 4$ sets, we generate a pool of random partitions $U$ with $r$ ranging from 2 to 10 sets. Then, we use NMI$_q(U,V)$ to select the closest partition to the reference $V$. The plot at the bottom of Figure \ref{fig:selbias} shows the probability of selection of a partition $U$ with $r$ sets using NMI$_q$ computed on 5000 simulations.
\begin{figure}[h]
\begin{minipage}{0.32\textwidth}
\includegraphics[scale=.65]{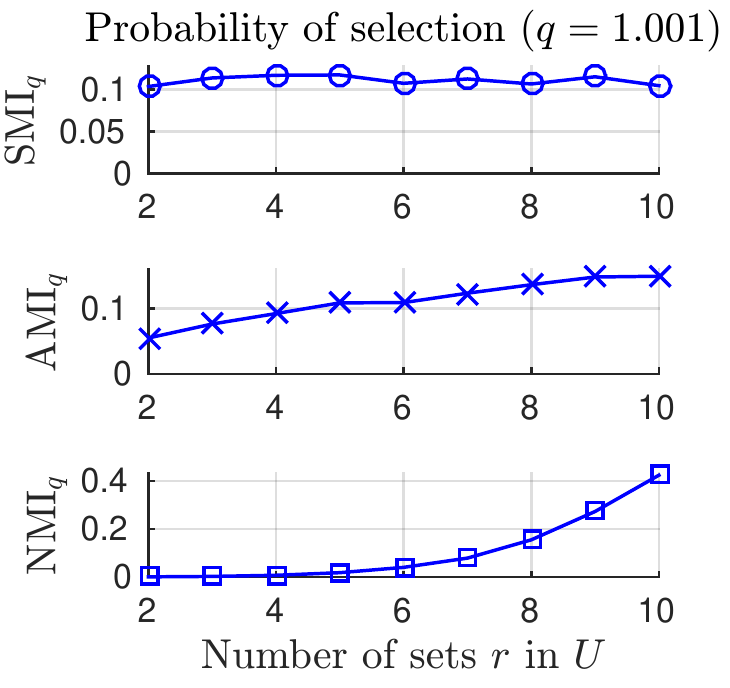}
\end{minipage}
\begin{minipage}{0.32\textwidth}
\includegraphics[scale=.65]{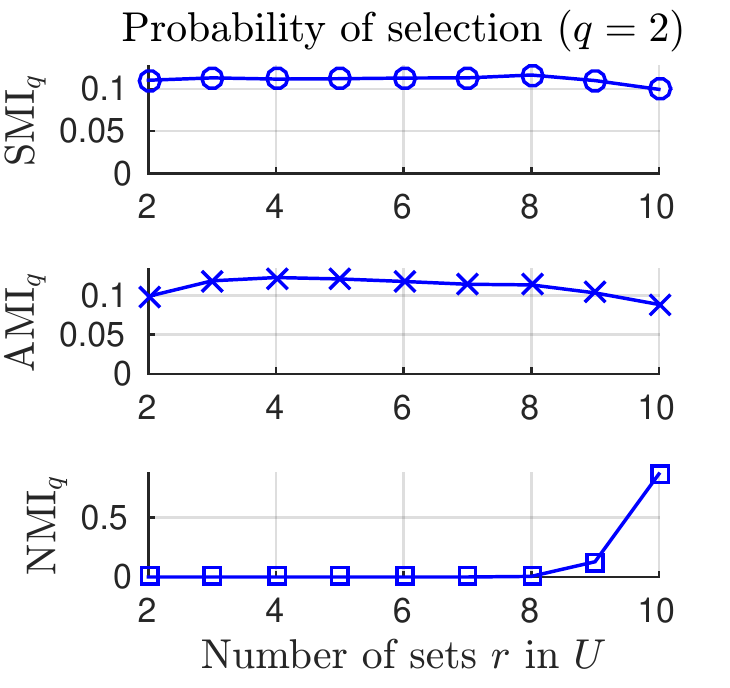}
\end{minipage}
\begin{minipage}{0.33\textwidth}
\includegraphics[scale=.65]{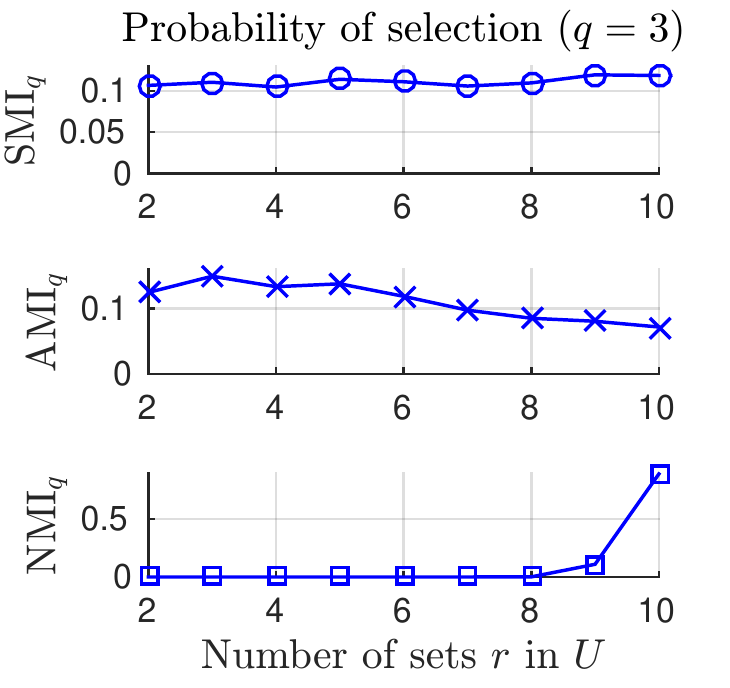}
\end{minipage}
\caption{Selection bias towards partitions with different $r$ when compared to a reference $V$. The probability of selection should be uniform when partitions are random. Using SMI$_q$ we achieve close to uniform probability of selection for $q$ equal to $1.001$, $2$ and $3$ respectively.
}\label{fig:selbias}
\end{figure}
We do not expect any partition to be the best given that they are all generated at random: \emph{i.e.,\ the plot is expected to be flat if a measure is unbiased}. Nonetheless, we see that there is a clear bias towards partitions with 10 sets if we use NMI$_q$ with $q$ respectively equal to 1.001, 2, or 3. We can see that the use of the adjusted measures such as AMI$_q$ helps in decreasing this bias, in particular when $q = 2$. On this experiment when $q=2$, baseline adjustment seems to be effective in decreasing the selection bias because the variance AMI$_2 = \mbox{ARI}$ is almost constant. However for all $q$, using SMI$_q$ we obtain close to uniform probability of selection of each random partition $U$.
 
\subsection{Large Number of Objects} \label{sec:largeobjstd}

It is likely to expect that the variance of generalized IT measures decreases when partitions are generated on a large number of objects $N$.
Here we prove a general result about measures of the family $\mathcal{N}_\phi$.
\begin{restatable}{lem}{lemmavarg} \label{lemma:varg}
If $S(U,V) \in \mathcal{N}_\phi$, then $\lim_{N\rightarrow +\infty} \mbox{\emph{Var}}(S(U,V))=0$.
\end{restatable}
\noindent Given that generalized IT measures belong in the family $\mathcal{N}_\phi$, we can prove the following:
\begin{restatable}{thm}{largeNvarMI} \label{thm:largevarMI}
It holds true that: 
\begin{equation}
\lim_{N \rightarrow +\infty} \mbox{\emph{Var}}(H_q(U,V)) = \lim_{N \rightarrow +\infty} \mbox{\emph{Var}}(\mbox{\emph{MI}}_q(U,V)) = \lim_{N \rightarrow +\infty} \mbox{\emph{Var}}(\mbox{\emph{VI}}_q(U,V)) = 0
\end{equation}
\end{restatable}
\noindent Therefore, SMI$_q$ attains very large values when $N$ is large. In practice of course, $N$ is finite, so the use of SMI$_q$ is beneficial. However, it is less important to correct for selection bias if the number of objects $N$ is big with regards to the number of cells in the contingency table in Table \ref{tbl:contingency}: i.e., when $\frac{N}{r \cdot c}$ is large. Indeed, when the number of objects is large AMI$_q$ might be sufficient to avoid selection bias and any test for independence between partitions has high power. In this scenario, SMI$_q$ is not needed and AMI$_q$ might be preferred as it can be computed more efficiently.

\section{Conclusion}

In this paper, we computed the exact expected value and variance of measures of the family $\mathcal{L}_\phi$, which contains generalized IT measures. We also showed how the expected value for measures $S \in \mathcal{N}_\phi$ can be computed for large $N$. Using these statistics, we proposed AMI$_q$ and SMI$_q$ to adjust generalized IT measures both for baseline and for selection bias. AMI$_q$ is a further generalization of well known measures for clustering comparisons such as ARI and AMI. This analysis allowed us to provide guidelines for their best application in different scenarios. In particular ARI might be used as external validation index when the reference clustering shows big equal sized clusters. AMI can be used when the reference clustering is unbalanced and there exist small clusters. The standardized SMI$_q$ can instead be used to correct for selection bias among many possible candidate clustering solutions when the number of objects is small. Furthermore, it can also be used to test the independence between two partitions. All code has been made available online\footnote{ \texttt{https://sites.google.com/site/adjgenit/} }.


\appendix

\newpage

\section{Theorem Proofs}

\propviri*
\begin{proof}
\begin{align*}
\mbox{VI}_q(U,V) &= 2H_q(U,V) - H_q(U) - H_q(V) \\
&= \frac{2}{q - 1}\Big(1 - \sum_{i=1}^r \sum_{j=1}^c \Big( \frac{n_{ij}}{N} \Big)^q \Big)
-\frac{1}{q - 1}\Big(1 - \sum_{i=1}^r \Big( \frac{a_{i}}{N} \Big)^q \Big)
-\frac{1}{q - 1}\Big(1 - \sum_{j=1}^c \Big( \frac{b_{j}}{N} \Big)^q \Big)\\
&= \frac{1}{(q - 1)N^q}\Big(\sum_{i=1}^r a_i^q + \sum_{j=1}^c b_j^q - 2\sum_{i=1}^r \sum_{j=1}^c n_{ij}^q \Big)\\
\end{align*}
When $q = 2$, $\mbox{VI}_2(U,V) = \frac{1}{N^2} (\sum_i a_i^2 + \sum_j b_j^2 - 2\sum_{i,j} n_{ij}^2 ) = \frac{1}{N^2}\mbox{MK}(U,V) = \frac{N-1}{N}(1-\mbox{RI}(U,V))$.
\end{proof}

\lememalpha*
\begin{proof}
The expected value of $S(U,V)$ according to the hypergeometric model of randomness is $E[S(U,V)] = \sum_{\mathcal{M}} S(\mathcal{M})P(\mathcal{M})$ where $\mathcal{M}$ is a contingency table generated via permutations. This is reduced to $E[S(U,V)] = \sum_{\mathcal{M}} (\alpha + \beta  \sum_{ij} \phi_{ij}(n_{ij}) )P(\mathcal{M}) = \alpha + \beta\sum_{\mathcal{M}} \sum_{ij} \phi_{ij}(n_{ij})P(\mathcal{M})$. Because of linearity of the expected value, it is possible to swap the summation over $\mathcal{M}$ and the one over cells obtaining $\alpha + \beta \sum_{ij} \sum_{n_{ij}} \phi_{ij}(n_{ij}) P(n_{ij}) = \alpha + \beta \sum_{ij} E[\phi_{ij}(n_{ij})]$ where $n_{ij}$ is a hypergeometric distribution with the marginals $a_i$, $b_j$, and $N$ as parameters, i.e.\ $n_{ij} \sim \mbox{Hyp}(a_i,b_j,N)$.
\end{proof}


\thmvibeta*
\begin{proof}
The results easily follow from Lemma \ref{lemma:emalpha} and the hypothesis of fixed marginals.
\end{proof}

\adjan*
\begin{proof}
The using the upper bound $\frac{1}{2}(H_q(U) + H_q(V))$ to $\mbox{MI}_q$, $\mbox{AMI}_q$ and $\mbox{AVI}_q$ are equivalent. Therefore we compute $\mbox{AVI}_q$. The denominator is equal to $E[\mbox{VI}_q] = \frac{2}{(q-1)N^q} \Big( \frac{1}{2}(\sum_i a_i^q + \sum_j b_j^q) - \sum_{i,j} E[n_{ij}^q] \Big)$. The numerator is instead $\frac{2}{(q-1)N^q} \Big( \sum_{ij} n_{ij}^q - \sum_{i,j} E[n_{ij}^q] \Big)$.
\end{proof}

\adjspec*
\begin{proof}
Point \emph{i)} follows from the limit of the $q$-entropy when $q \rightarrow 1$. Point \emph{ii)} follows from:
\[
 \mbox{AVI}_2 = \frac{E[\mbox{VI}_2] - \mbox{VI}_2}{E[\mbox{VI}_2] - \min{\mbox{VI}_2}}
=  \frac{ \frac{N-1}{N} ( \mbox{RI} - E[\mbox{RI}])}{ \frac{N-1}{N} ( \max{\mbox{RI}} - E[\mbox{RI}]) } = \mbox{ARI}
\]
\end{proof}

\complexp*
\begin{proof}
The computation of $P(n_{ij})$ where $n_{ij}$ is a hypergeometric distribution $\mbox{Hyp}(a_i,b_j,N)$ is linear in $N$. However, the computation of the expected value $E[n_{ij}^q] = \sum_{n_{ij}} n_{ij}^q P(n_{ij})$ can exploit the fact that $P(n_{ij})$ are computed iteratively: $P(n_{ij}+1) = P(n_{ij}) \frac{(a_i - n_{ij}) (b_j- n_{ij}) } { (n_{ij}+1)(N - a_i - b_j + n_{ij} + 1) }$. We compute $P(n_{ij})$ only for \ $\max{ \{0, a_i + b_j -N \} }$. In both cases $P(n_{ij})$ can be computed in $O(\max{ \{a_i,b_j \} })$. We can compute all other probabilities iteratively as shown above in constant time. Therefore:
\begin{align*}
\sum_{i=1}^{r} \sum_{j=1}^c \left( O(\max{ \{a_i,b_j \} }) + \sum_{n_{ij} = 0}^{\min{ \{a_i,b_j \} }} O(1) \right) &=  \sum_{i=1}^{r} \sum_{j=1}^c O(\max{ \{a_i,b_j \} }) =  \sum_{i=1}^{r} O(\max{ \{c a_i,N \} })\\
&= O(\max{ \{cN,rN\} }) = O(N \cdot \max{ \{c,r\} })
\end{align*}
\end{proof}

\largephi*
\begin{proof}
$S(U,V)$ can be written as $\phi(\frac{n_{11}}{N},\dots,\frac{n_{ij}}{N},\dots,\frac{n_{rc}}{N})$. Let $\mathbf{X} = (X_1,\dots,X_{rc}) =(\frac{n_{11}}{N},\dots,\frac{n_{ij}}{N},\dots,\frac{n_{rc}}{N})$ be a vector of $rc$ random variables where $n_{ij}$ is a hypergeometric distribution with the marginals as parameters: $a_i$, $b_j$ and $N$. The expected value of $\frac{n_{ij}}{N}$ is $E[\frac{n_{ij}}{N}] = \frac{1}{N} \frac{a_i b_j}{N}$. Let $\boldsymbol\mu = (\mu_1,\dots,\mu_{rc})=(E[X_1],\dots,E[X_{rc}]) = (\frac{a_1}{N}\frac{b_1}{N},\dots,\frac{a_i}{N}\frac{b_j}{N},\dots,\frac{a_r}{N}\frac{b_c}{N})$ be the vector of the expected values. The Taylor approximation of $S(U,V) = \phi(\mathbf{X})$ around $\boldsymbol\mu$ is:
\[
\phi(\mathbf{X}) \simeq \phi(\boldsymbol\mu) + \sum_{t=1}^{rc}(X_t - \mu_t)\frac{\partial \phi}{\partial X_t}
+ \frac{1}{2} \sum_{t=1}^{rc} \sum_{s=1}^{rc}(X_t - \mu_t)(X_s - \mu_s)\frac{\partial^2 \phi}{\partial X_t \partial X_s} + \dots
\]
Its expected value is (see Section 4.3 of~\citep{Ang2006}):
\[
E[\phi(\mathbf{X})] \simeq \phi(\boldsymbol\mu) + \frac{1}{2} \sum_{t=1}^{rc} \sum_{s=1}^{rc}\mbox{Cov}(X_t,X_s)\frac{\partial^2 \phi}{\partial X_t \partial X_s} + \dots
\]
We just analyse the second order remainder given that it dominates the higher order ones. Using the Cauchy-Schwartz inequality we have that $|\mbox{Cov}(X_t,X_s)| \leq \sqrt{\mbox{Var}(X_t)\mbox{Var}(X_s)}$. Each $X_t$ and $X_s$ is equal to $\frac{n_{ij}}{N}$ for some indexes $i$ and $j$.  The variance of each $X_t$ and $X_s$ is therefore equal to $\mbox{Var}(\frac{n_{ij}}{N}) = \frac{1}{N^2} \frac{a_i b_j}{N} \frac{N - a_i}{N} \frac{N- b_j}{N - 1}$. When the number of records is large also the marginals increase: $N \rightarrow + \infty \Rightarrow a_i \rightarrow + \infty$, and $b_j \rightarrow + \infty$ $\forall i,j$. However because of the permutation model, all the fractions $\frac{a_i}{N}$ and $\frac{b_j}{N}$ stay constant $\forall i,j$. Therefore, also $\boldsymbol\mu$ is constant. However, at the limit of large $N$, the variance of $\frac{n_{ij}}{N}$ tends to 0: $\mbox{Var}\Big(\frac{n_{ij}}{N}\Big) = \frac{1}{N} \frac{a_i}{N} \frac{b_j}{N} \Big( 1 - \frac{a_i}{N} \Big) \Big( 1 +\frac{1}{N-1} - \frac{b_j}{N} \Big) \rightarrow 0$. Therefore, at large $N$:
\[
E[\phi(\mathbf{X})] \simeq \phi(\boldsymbol\mu) = \phi\Big( \frac{a_1}{N}\frac{b_1}{N},\dots,\frac{a_i}{N}\frac{b_j}{N},\dots,\frac{a_r}{N}\frac{b_c}{N}\Big)
\]
\end{proof}

\largeNmean*
\begin{proof} $E[H_q(U,V)] = \frac{1}{q - 1} \Big( 1 - \sum_{ij}  E \Big[ \Big( \frac{n_{ij}}{N} \Big)^q \Big] \Big)$ and according to Lemma \ref{lemma:largephi} for large $N$: $E[H_q(U,V)] \simeq \frac{1}{q-1} \Big( 1 -  \sum_{ij}  \Big( \frac{ a_i}{N} \frac{b_j}{ N } \Big)^q \Big) = \frac{1}{q-1} \Big( 1 -  \sum_{i}  \Big( \frac{ a_i}{N} \Big)^q \sum_j \Big( \frac{b_j}{ N } \Big)^q \Big)$. If we add an subtract $1 - \sum_i  \Big( \frac{ a_i }{ N } \Big)^q - \sum_j  \Big( \frac{ b_j }{ N } \Big)^q$ in the parenthesis above:
\begin{align*}
E[H_q(U,V)] &\simeq \frac{1}{q -1} \Big( 1 -  \sum_i  \Big( \frac{ a_i }{ N } \Big)^q \sum_j  \Big( \frac{ b_j}{ N } \Big)^q \\
&+ 1 - \sum_i  \Big( \frac{ a_i }{ N } \Big)^q - \sum_j  \Big( \frac{ b_j }{ N } \Big)^q  \\
&- 1 + \sum_i  \Big( \frac{ a_i }{ N } \Big)^q + \sum_j  \Big( \frac{ b_j }{ N } \Big)^q \Big)\\
&= \frac{1}{q-1} \Big( 1 - \sum_{i} \Big( \frac{ a_i }{ N } \Big)^q \Big)
+ \frac{1}{q-1} \Big( 1 - \sum_{j} \Big( \frac{ b_j }{ N } \Big)^q \Big)  \\
&+  \frac{1}{q-1} \Big(  - 1 -  \sum_i  \Big( \frac{ a_i }{ N } \Big)^q \sum_j  \Big( \frac{ b_j}{ N } \Big)^q + \sum_{i} \Big( \frac{ a_i }{ N } \Big)^q + \sum_{j} \Big( \frac{ b_j }{ N } \Big)^q \Big)\\
&= H_q (U) + H_q (V) 
+\frac{1}{q-1} \Big( \Big(1 - \sum_{i} \Big( \frac{ a_i }{ N } \Big)^q \Big) \Big( \sum_{j} \Big( \frac{ b_j }{ N } \Big)^q \Big) \Big) \\
&= H_q(U) + H_q(V) - (q-1) H_q (U)H_q (V)
\end{align*}
Point \emph{ii)} and \emph{iii)} follow from Equations \eqref{eq:mi_beta} and \eqref{eq:vi_beta}.
\end{proof}

\lemanalalpha*
\begin{proof}
The proof follows Theorem 1 proof in~\cite{Romano2014}. Using the properties of the variance we can show that $\mbox{Var}(S(U,V)) = \beta^2\mbox{Var}( \sum_{ij} \phi_{ij}(n_{ij})) = \beta^2\Big( E[(\sum_{ij} \phi_{ij}(n_{ij}))^2] - (\sum_{ij} E[\phi_{ij}(n_{ij})])^2\Big)$. $( E[ \sum_{ij} \phi_{ij}(n_{ij})])^2 = ( \sum_{ij} E[\phi_{ij}(n_{ij})])^2$ can be computed using Eq.\ \eqref{eq:ephinij}. The first term in the sum is instead:
\begin{align*}
&E[ (\sum_{ij} \phi_{ij}(n_{ij}))^2 ] = \sum_{ij}\sum_{i'j'} E[\phi_{ij}(n_{ij})\phi_{i'j'}(n_{i'j'})] = \sum_{ij}\sum_{i'j'} \sum_{n_{ij}} \sum_{n_{i'j'}}\phi_{ij}(n_{ij})\phi_{i'j'}(n_{i'j'}) P(n_{ij},n_{i'j'}) \\
\end{align*}
We cannot find the exact form of the joint probability $P(n_{ij},n_{i'j'})$ thus we rewrite it as $P(n_{ij})P(n_{i'j'}|n_{ij}) = P(n_{ij})P(\tilde{n}_{i'j'})$. The random variable $n_{ij}$ is an hypergeometric distribution that simulates the experiment of sampling without replacement the $a_i$ objects in the set $u_i$ from a total of $N$ objects. Sampling one of the $b_j$ objects from $v_j$ is defined as success: $n_{ij} \sim \mbox{Hyp}(a_i,b_j,N)$. The random variable $\tilde{n}_{i'j'}$ has a different distribution depending on the possible combinations of indexes $i,i',j,j'$. Thus $E[ (\sum_{ij} \phi_{ij}(n_{ij}))^2 ]$ is equal to:
\[\sum_{ij} \sum_{n_{ij}} \sum_{i'j'} \sum_{n_{i'j'}}\phi_{ij}(n_{ij}) \phi_{i'j'}(n_{i'j'}) P(n_{ij},n_{i'j'}) =  \sum_{ij} \sum_{n_{ij}} \phi_{ij}(n_{ij}) P(n_{ij})\sum_{i'j'} \sum_{\tilde{n}_{i'j'}}\phi_{i'j'}(\tilde{n}_{i'j'}) P(\tilde{n}_{i'j'})
\]
which, by taking care of all possible combinations of $i,i',j,j'$, is equal to :
\begin{align}
\sum_{ij} \sum_{ n_{ij}}   \phi_{ij}(n_{ij}) P(n_{ij}) \cdot \Bigg[&
\sum_{i'=i,j'=j} \sum_{\tilde{n}_{ij}} \phi_{ij}(\tilde{n}_{ij}) P(\tilde{n}_{ij}) 
&+ \sum_{i'\neq i, j'=j} \sum_{\tilde{n}_{i'j} } \phi_{i'j}(\tilde{n}_{i'j}) P(\tilde{n}_{i'j}) \\
&+ \sum_{i'=i,j'\neq j}  \sum_{\tilde{n}_{ij'} } \phi_{ij'}(\tilde{n}_{ij'}) P(\tilde{n}_{ij'}) 
&+ \sum_{i' \neq i, j' \neq j} \sum_{\tilde{n}_{i'j'} } \phi_{i'j'}(\tilde{n}_{i'j'}) P(\tilde{n}_{i'j'}) \Bigg] \label{eq:together}
\end{align}
\textbf{Case 1:} {$i'=i \wedge j'=j $}

$P(\tilde{n}_{ij}) = 1$ if and only if $\tilde{n}_{ij} = n_{ij}$ and $0$ otherwise. This case produces the first term $\phi_{ij}(n_{ij})$ enclosed in square brackets.
\\[.5cm]
\textbf{Case 2:} {$i'=i \wedge j' \neq j $}

In this case, the possible successes are the objects from the set $v_{j'}$. We have already sampled $n_{ij}$ objects and we are sampling from the whole set of objects excluding the set $v_j$. Thus, $\tilde{n}_{ij'}\sim \mbox{Hyp}(a_i - n_{ij}, b_{j'}, N - b_j)$.
\\[.5cm]
\textbf{Case 3:} {$i' \neq i \wedge j' = j $}

This case is  symmetric to the previous one where $a_{i'}$ is now the possible number of successes. Therefore $\tilde{n}_{i'j} \sim \mbox{Hyp}(b_j - n_{ij}, a_{i'}, N - a_i)$.
\\[.5cm]
\textbf{Case 4:} {$i' \neq i \wedge j' \neq j $}

In order compute $P(\tilde{n}_{i'j'})$, we have to impose a further condition:
\[
P(\tilde{n}_{i'j'}) = \sum_{ \tilde{n}_{ij'} } P(\tilde{n}_{i'j'}|\tilde{n}_{ij'})P(\tilde{n}_{ij'}) = \sum_{ \tilde{n}_{ij'} } P(\tilde{\tilde{n}}_{i'j'})P(\tilde{n}_{ij'})
\]
We are considering sampling the $a_{i'}$ objects in $u_{i'}$ from the whole set of objects excluding the $a_{i}$ objects from $u_i$. Just knowing that $n_{ij}$ objects have already been sampled from $u_i$ does not allow us to know how many objects from $v_{j'}$ have also been sampled. If we know that $n_{ij'}$ are the number of objects sampled from $v_{j'}$, we know there are $b_{j'} - n_{ij'}$ possible successes and thus $\tilde{n}_{i'j'}|\tilde{n}_{ij'} = \tilde{\tilde{n}}_{i'j'} \sim \mbox{Hyp}(a_{i'},b_{j'} - \tilde{n}_{ij'}, N - a_i)$. So the last two terms in Eq.\ \eqref{eq:together} can be put together:
\begin{align*}
&\sum_{i'=i,j'\neq j}  \sum_{\tilde{n}_{ij'} } \phi_{ij'}(\tilde{n}_{ij'}) P(\tilde{n}_{ij'}) + \sum_{i' \neq i, j' \neq j} \sum_{\tilde{n}_{i'j'} } \phi_{i'j'}(\tilde{n}_{i'j'}) P(\tilde{n}_{i'j'}) \\
&= \sum_{j'\neq j}  \sum_{\tilde{n}_{ij'} } \phi_{ij'}(\tilde{n}_{ij'}) P(\tilde{n}_{ij'}) + \sum_{i' \neq i, j' \neq j} \sum_{\tilde{n}_{i'j'} }\phi_{i'j'}(\tilde{n}_{i'j'}) \sum_{ \tilde{n}_{ij'} } P(\tilde{n}_{i'j'}|\tilde{n}_{ij'})P(\tilde{n}_{ij'}) \\
&= \sum_{j'\neq j}  \sum_{\tilde{n}_{ij'} } \phi_{ij'}(\tilde{n}_{ij'}) P(\tilde{n}_{ij'}) + \sum_{i' \neq i, j' \neq j} \sum_{\tilde{\tilde{n}}_{i'j'} }\phi_{i'j'}(\tilde{\tilde{n}}_{i'j'}) \sum_{ \tilde{n}_{ij'} } P(\tilde{\tilde{n}}_{i'j'})P(\tilde{n}_{ij'}) \\
&= \sum_{j'\neq j}  \sum_{\tilde{n}_{ij'} } P(\tilde{n}_{ij'})\phi_{ij'}(\tilde{n}_{ij'}) + \sum_{j' \neq j} \sum_{ \tilde{n}_{ij'} }P(\tilde{n}_{ij'}) \sum_{i' \neq i} \sum_{\tilde{\tilde{n}}_{i'j'} }\phi_{i'j'}(\tilde{\tilde{n}}_{i'j'})  P(\tilde{\tilde{n}}_{i'j'})\\
&= \sum_{j'\neq j}  \sum_{\tilde{n}_{ij'} } P(\tilde{n}_{ij'}) \Bigg( \phi_{ij'}(\tilde{n}_{ij'}) + \sum_{i' \neq i} \sum_{\tilde{\tilde{n}}_{i'j'} }\phi_{i'j'}(\tilde{\tilde{n}}_{i'j'})  P(\tilde{\tilde{n}}_{i'j'}) \Bigg)
\end{align*}
By putting everything together we get that $E[( \sum_{ij} \phi_{ij}(n_{ij}))^2]$ is equal to:
\begin{align*}
\sum_{ij} \sum_{ n_{ij}}   \phi(n_{ij}) P(n_{ij}) \cdot \Bigg[&  \phi_{ij}(n_{ij}) + \sum_{i'\neq i} \sum_{\tilde{n}_{i'j} }\phi_{i'j}(\tilde{n}_{i'j}) P(\tilde{n}_{i'j}) + \nonumber\\
& + \sum_{j'\neq j}  \sum_{\tilde{n}_{ij'} } P(\tilde{n}_{ij'}) \Bigg( \phi_{ij'}(\tilde{n}_{ij'}) + \sum_{i' \neq i} \sum_{\tilde{\tilde{n}}_{i'j'} }\phi_{i'j'}(\tilde{\tilde{n}}_{i'j'})P(\tilde{\tilde{n}}_{i'j'}) \Bigg) \Bigg]
\end{align*}
\end{proof}


\thmvarsvi*
\begin{proof}
The results follow from Lemma \ref{lem:analalpha}, the hypothesis of fixed marginals and properties of the variance.
\end{proof}

\thmsvianal*
\begin{proof}
\\
For $\mbox{SMI}_q$: the numerator is equal to $H_q(U,V) - E[H_q(U,V)] = \frac{1}{(q-1)N^q} \Big( \sum_{ij} n_{ij}^q - \sum_{i,j} E[n_{ij}^q] \Big)$. According Theorem \ref{thm:varsvi}, the denominator is instead:
\[
\sqrt{\mbox{Var}(\mbox{MI}_q(U,V))} = \sqrt{\mbox{Var}(\mbox{H}_q(U,V))} = \frac{1}{(q -1)N^{q}} \sqrt{ E[ (\sum_{ij} n_{ij}^q)^2 ] - ( E[ \sum_{ij} n_{ij}^q ])^2 }.
\]
For $\mbox{SVI}_q$: the numerator is equal to $2H_q(U,V) - 2E[H_q(U,V)] = \frac{2}{(q-1)N^q} \Big( \sum_{ij} n_{ij}^q - \sum_{i,j} E[n_{ij}^q] \Big)$. According Theorem \ref{thm:varsvi}, the denominator is instead:
\[
\sqrt{\mbox{Var}(\mbox{VI}_q(U,V))} = \sqrt{4\mbox{Var}(\mbox{H}_q(U,V))} = \frac{2}{(q -1)N^{q}} \sqrt{ E[ (\sum_{ij} n_{ij}^q)^2 ] - ( E[ \sum_{ij} n_{ij}^q ])^2 }.
\]
Therefore, $\mbox{SMI}_q$ and $\mbox{SVI}_q$ are equivalent.
\end{proof}

\corspec*
\begin{proof}
Point \emph{i)} follows from the limit of the $q$-entropy when $q \rightarrow 1$ and the linear relation to of $G$-statistic to MI: $G = 2N\mbox{MI}$. Point \emph{ii)} follows from:
\[
 \mbox{SVI}_2 = \frac{E[\mbox{VI}_2] - \mbox{VI}_2}{\sqrt{ \mbox{Var}(\mbox{VI}_2)}}
=  \frac{ \frac{N-1}{N} ( \mbox{RI} - E[\mbox{RI}])}{ \frac{N-1}{N} \sqrt{ \mbox{Var}(\mbox{RI})}} = \mbox{SRI}
\]
\end{proof}

\complvar*
\begin{proof}
Each summation in Eq.\ \eqref{eq:enij2alpha} can be bounded above by the maximum value of the cell marginals and each sum can be done in constant time. The last summation in Eq.\ \eqref{eq:enij2alpha} is:
\begin{align*}
\sum_{j'=1}^c \sum_{\tilde{n}_{ij'} = 0}^{\max{ \{a_i,b_{j'} \} }} \sum_{i'=1}^{r} \sum_{\tilde{\tilde{n}}_{i'j'} = 0}^{\max{ \{a_{i'},b_{j'} \} }} O(1) &= \sum_{j'=1}^c \sum_{\tilde{n}_{ij'} = 0}^{\max{ \{a_i,b_{j'} \} }} O( \max{ \{N,r b_{j'} \} }) \\ &=\sum_{j'=1}^c O( \max{ \{a_i N,a_i r b_{j'}, b_{j'} N , r b_{j'}^2 \} }) \\
&= O( \max{ \{c a_i N,a_i r N, r N^2 \} })
\end{align*}
The above term is thus the computational complexity of the inner loop. Using the same machinery one can prove that:
\begin{align*}
\sum_{j=1}^c \sum_{i=1}^{r} \sum_{n_{ij} = 0}^{\max{ \{a_{i},b_{j} \} }} O( \max{ \{c a_i N,a_i r N, r N^2 \} }) =O( \max{ \{c^2 N^3, rcN^3\} })=O( N^3c \cdot \max{ \{c , r\} })
\end{align*}
\end{proof}

\proppval*
\begin{proof}
Let MI$^0_q$ be the random variable under the null hypothesis of independence between partitions associated to the test statistic MI$_q(U,V)$. The $p$-value is defined as:
\begin{align*}
p\text{-value} &= P\Big( \mbox{MI}^0_q \geq \mbox{MI}_q(U,V) \Big) = P\Big( \mbox{MI}^0_q  - E[\mbox{MI}_q(U,V)] \geq \mbox{MI}_q(U,V) - E[\mbox{MI}_q(U,V)] \Big)\\
&= P\Bigg( \frac{\textup{MI}^0_q  - E[\textup{MI}_q(U,V)]}{\sqrt{ \textup{Var}(\textup{MI}_q(U,V))}} \geq \frac{\textup{MI}_q(U,V)  - E[\textup{MI}_q(U,V)]}{\sqrt{ \textup{Var}(\textup{MI}_q(U,V))}} \Bigg)\\
&= P\Bigg( \frac{\textup{MI}^0_q  - E[\textup{MI}_q(U,V)]}{\sqrt{ \textup{Var}(\textup{MI}_q(U,V))}} \geq \textup{SMI}_q(U,V) \Bigg)\\
\end{align*}
Let $Z$ be the standardized random variable $\frac{\textup{MI}^0_q  - E[\textup{MI}_q(U,V)]}{\sqrt{ \textup{Var}(\textup{MI}_q(U,V))}}$, then using the one side Chebyshev's inequality also known as the Cantelli's inequality~\citep{RossBook}:
\begin{align*}
p\text{-value} = P(Z \geq \textup{SMI}_q(U,V)) < \frac{1}{1 + \Big( \textup{SMI}_q(U,V) \Big)^2}
\end{align*}
\end{proof}

\lemmavarg*
\begin{proof}
Let $\mathbf{X} = (X_1,\dots,X_{rc}) =(\frac{n_{11}}{N},\dots,\frac{n_{ij}}{N},\dots,\frac{n_{rc}}{N})$ be a vector of $rc$ random variables where $n_{ij}$ is a hypergeometric distribution with the marginals as parameters: $a_i$, $b_j$ and $N$. Using the Taylor approximation~\citep{Ang2006} of $S(U,V) = \phi(\mathbf{X})$, it is possible to show that:
\[
\mbox{Var}(\phi(\mathbf{X}))\simeq \sum_{t=1}^{rc} \sum_{s=1}^{rc} \mbox{Cov}(X_t,X_s) \frac{\partial \phi}{\partial X_t} \frac{\partial \phi}{\partial X_s} + \dots
\]
Using the Cauchy-Schwartz inequality we have that $|\mbox{Cov}(X_t,X_s)| \leq \sqrt{\mbox{Var}(X_t)\mbox{Var}(X_s)}$. Each $X_t$ and $X_s$ is equal to $\frac{n_{ij}}{N}$ for some indexes $i$ and $j$.  The variance of each $X_t$ and $X_s$ is therefore equal to $\mbox{Var}(\frac{n_{ij}}{N}) = \frac{1}{N^2} \frac{a_i b_j}{N} \frac{N - a_i}{N} \frac{N- b_j}{N - 1}$. When the number of records is large also the marginals increase: $N \rightarrow + \infty \Rightarrow a_i \rightarrow + \infty$, and $b_j \rightarrow + \infty$ $\forall i,j$. However because of the permutation model, all the fractions $\frac{a_i}{N}$ and $\frac{b_j}{N}$ stay constant $\forall i,j$. Therefore, at the limit of large $N$, the variance of $\frac{n_{ij}}{N}$ tends to 0: $\mbox{Var}\Big(\frac{n_{ij}}{N}\Big) = \frac{1}{N} \frac{a_i}{N} \frac{b_j}{N} \Big( 1 - \frac{a_i}{N} \Big) \Big( 1 +\frac{1}{N-1} - \frac{b_j}{N} \Big) \rightarrow 0$ and thus $\mbox{Var}(\phi(\mathbf{X}))$ tends to $0$.
\end{proof}

\largeNvarMI*
\begin{proof}
Trivially follows from Lemma \ref{lemma:varg}.
\end{proof}

\newpage

\bibliography{mandb} 

\end{document}